\documentclass{article}
\usepackage{enumitem}

\PassOptionsToPackage{numbers, compress}{natbib}

\usepackage[preprint]{neurips_2026}
\usepackage[utf8]{inputenc}
\usepackage[T1]{fontenc}
\usepackage{hyperref}
\usepackage{url}
\usepackage{booktabs}
\usepackage{multirow}
\usepackage{amsfonts}
\usepackage{amsmath}
\usepackage{amssymb}
\usepackage{amsthm}
\usepackage{bm}

\newtheorem{theorem}{Theorem}
\newtheorem{lemma}{Lemma}
\newtheorem{proposition}{Proposition}

\newtheorem{assumption}{Assumption}
\usepackage{nicefrac}
\usepackage{microtype}
\usepackage{xcolor}
\usepackage{graphicx}
\usepackage{wrapfig}
\usepackage{subcaption}
\usepackage{algorithm}
\usepackage{algpseudocode}
\usepackage{fancyvrb}

\newcommand{\methodname}{\texttt{CreFlow}}
\newcommand{\xhat}{\hat{\mathbf{x}}}

\newcommand{\actphi}{\phi}
\newcommand{\vold}{v_\theta^{\text{old}}}
\newcommand{\vcur}{v_\theta}
\newcommand{\Lnft}{\mathcal{L}_{\text{NFT}}}
\newcommand{\Lcr}{\mathcal{L}_{\text{CR}}}
\newcommand{\Lcore}{\mathcal{L}_{\text{CreFlow}}}
\newcommand{\odotm}{\,\odot\,}

\newcommand{\xbarpos}{\bar{\mathbf{x}}_{0}^{+}}

\title{CreFlow: Corrective Reflow for Sparse-Reward Embodied Video Diffusion RL}

\author{%
  Zhenyang Ni\textsuperscript{1} \And
  Yijiang Li\textsuperscript{2} \And
  Ruochen Jiao\textsuperscript{1} \And
  Simon Sinong Zhan\textsuperscript{1} \And
  Sipeng Chen\textsuperscript{1}
  \AND
  Zhenfei Yin\textsuperscript{3} \And
  Minshuo Chen\textsuperscript{1} \And
  Philip Torr\textsuperscript{3} \And
  Zhaoran Wang\textsuperscript{1} \And
  Qi Zhu\textsuperscript{1}
  \AND
  \normalfont
  \textsuperscript{1}Northwestern University \quad
  \textsuperscript{2}University of California, San Diego \quad
  \textsuperscript{3}University of Oxford
}

\begin{document}

\maketitle

\begin{abstract}
Video generation models trained on heterogeneous data with likelihood-surrogate objectives can produce visually plausible rollouts that violate physical constraints in embodied manipulation.
Although reinforcement-learning post-training offers a natural route to adapting VGMs, existing video-RL rewards often reduce each rollout to a low-level visual metric, whereas manipulation video evaluation requires logic-based verification of whether the rollout satisfies a compositional task specification.
To fill this gap, we introduce a compositional constraint-based reward model for post-training embodied video generation models, which automatically formulates task requirements as a composition of Linear Temporal Logic constraints, providing faithful rewards and localized error information in generated videos.
To achieve effective improvement in high-dimensional video generation using these reward signals, we further propose \methodname{}, a novel online RL framework with two key designs: i) a credit-aware NFT loss that confines the RL update to reward-relevant regions, preventing perturbations to unrelated regions during post-training; and ii) a corrective reflow loss that leverages within-group positive samples as an explicit estimate of the correction direction, stabilizing and accelerating training.
Experiments show that \methodname{} yields reward judgments better aligned with human and simulator success labels than existing methods and improves downstream execution success by 23.8 percentage points across eight bimanual manipulation tasks.
\end{abstract}

\section{Introduction}
\label{sec:intro}

A growing body of work has explored the integration of video generation models (VGMs) into robotic manipulation systems~\cite{unipi,lvp,dream2flow,bharadhwaj2024gen2act,vidar,tan2025anypos}.
Despite their impressive perceptual fidelity, videos synthesized by modern VGMs still suffer from physical implausibility in embodied manipulation.
This limitation stems from both data and objective mismatches: web-scale pretraining data include heterogeneous sources such as entertainment videos that may violate physical rules, while likelihood-surrogate objectives~\cite{vampo} treat pixels uniformly and favor global plausibility over task-critical robot-object interactions.
To mitigate these limitations, recent studies have used reinforcement learning to align general-purpose VGMs with specific reward targets, including visual quality~\cite{wang2026worldcompass}, scene-level 3D consistency~\cite{videogpa}, and action executability~\cite{wang2026eva}.

Despite these promising efforts, reward design for manipulation video RL remains substantially more demanding than generic video evaluation.
A generated video may appear globally coherent while still containing subtle pixel-level errors that will lead to inaccurate actions and compounding failures over time~\cite{vampo}; see Fig.~\ref{fig:teaser}. 
Targeting single aspects of video quality~\cite{wang2026worldcompass,wang2026eva}, existing video RL rewards are prone to missing such small but decisive violations, since the reward for evaluating manipulation videos is naturally compositional~\cite{emboalign}: multiple requirements must be satisfied simultaneously for task execution to succeed.
For example, in a put-object-into-cabinet-drawer task, the object must persist throughout the video, the robot motion must remain kinematically valid, and the object must be placed inside the target drawer with correct action causality (drawer should not open without a robot grasping it).
Moreover, even a faithful reward signal---whose verdict reflects downstream task success---is insufficient, as high-dimensional video rollouts are dominated by reward-irrelevant pixels and time intervals. Uniformly applying RL updates to the full video therefore wastes gradient on these areas, perturbing regions that are already plausible or unrelated to downstream execution.

\begin{figure}[t]
  \centering
  \includegraphics[width=\linewidth]{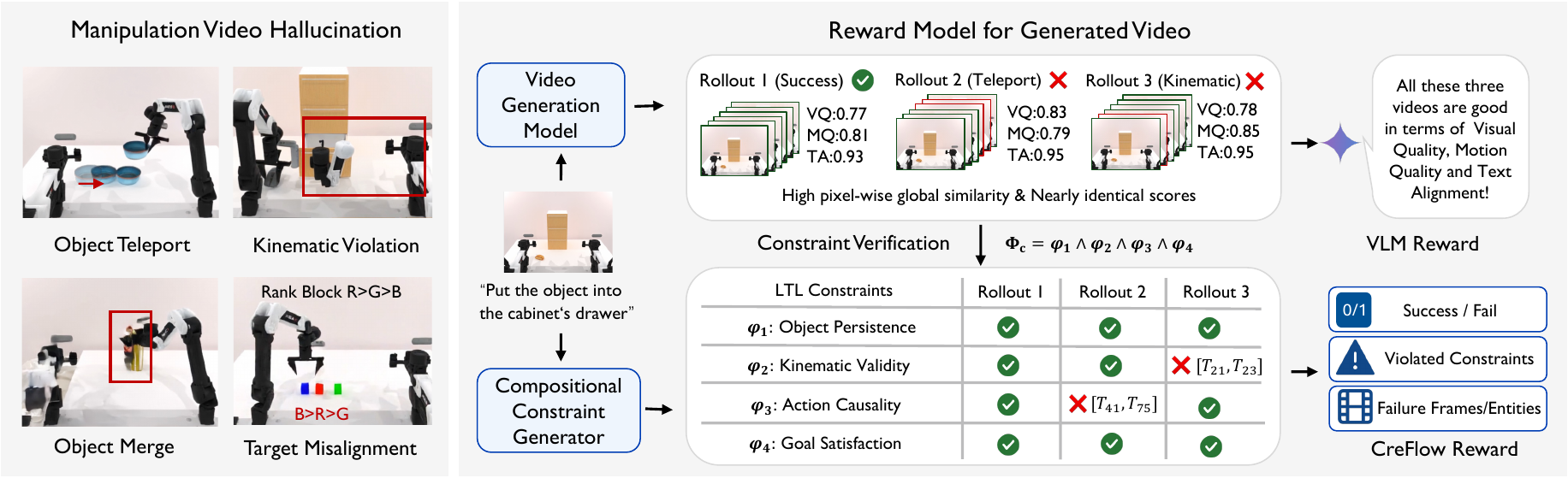}
  \caption{ Generated manipulation videos can be nearly indistinguishable under existing reward models~\cite{liu2025videoalign,wang2026worldcompass}, yet differ critically in downstream robot execution.
This ambiguity arises from two properties of manipulation videos: the decisive error may occupy only a small spatio-temporal region, and correctness depends on logic-based verification of compositional robot--object constraints rather than low-level visual quality alone. \methodname{} evaluates rollouts with compositional LTL constraints $\Phi_c$, producing a faithful binary success signal together with localized violation traces over constraints, frames, and entities, providing faithful and informative signal for RL post-training.}
  \vspace{-7mm}
  \label{fig:teaser}
\end{figure}

To fill this gap, we introduce a compositional constraint-based reward model for post-training embodied video generation models, which provides not only  faithful reward that evaluates downstream task success possibility, but also diagnostic signals that expose where and why a rollout fails.
To realize this, our core idea is to automatically formulate task success as a set of Linear Temporal Logic (LTL) constraints, where each constraint encodes an atomic requirement, and their composition jointly represents the holistic evaluation objective.
Each LTL constraint is evaluated over robot and object state traces, capturing the causal logic in the generated videos and providing spatio-temporal diagnostic information for locating the violation.
The constraint generation process is implemented by a coding agent, which leverages low-level vision modules to extract the states needed for constraint evaluation, including SAM3~\cite{sam3}, VLM~\cite{gemini}, and IDM~\cite{vidar,tan2025anypos}.

Following this strategy, we propose Corrective Reflow (\methodname{}), an online RL framework that converts compositional constraint violations into localized corrective objectives for manipulation video diffusion models.
\methodname{} introduces two key designs for turning sparse task failures into actionable supervision.
First, it aggregates the per-rollout violation traces in a group into a spatio-temporal mask over the responsible entities and frames, and applies a credit-aware NFT loss that confines the positive--negative velocity contrast to reward-relevant regions.
This prevents sparse task feedback from perturbing irrelevant pixels and keeps non-violating regions anchored to the pretrained video prior.
Second, for each failed rollout, \methodname{} aggregates the within-group successful rollouts generated under the same initial observation and instruction, and uses their  mean aggregation as an explicit corrective target rather than merely suppressing the generation probability of negative samples.
Together, these training objectives provide a more effective way to leverage both constraint-based reward feedback and relative information among rollouts generated under the same condition.

The contributions of this work are summarized as follows: i) We introduce a compositional constraint-based reward model that evaluates embodied video generation with automatically generated LTL constraints, providing not only faithful reward but also localized diagnostics for subtle rollout errors. ii)  We propose \methodname{}, a novel video RL framework that turns sparse compositional violations into a localized corrective estimator, combining a credit-aware NFT loss and a corrective reflow loss to deliver focused, low-variance updates on task-critical regions. iii) Extensive experiments on eight RoboTwin tasks show that \methodname{} yields reward judgments better aligned with human and simulator success labels than existing methods, leads to effective training to improve high dimensional video generation, and improves downstream execution success by 23.8 percentage points.
% \begin{itemize}[leftmargin=*]
% \item We introduce a compositional constraint-based reward model that evaluates embodied video generation with automatically generated LTL constraints, providing not only faithful reward but also localized diagnostics for subtle rollout errors.
% \item  We propose \methodname{}, a novel video RL framework that turns sparse compositional violations into a localized corrective estimator, combining a credit-aware NFT loss and a corrective reflow loss to deliver focused, low-variance updates on task-critical regions.
% \item  Extensive experiments on eight RoboTwin tasks show that \methodname{} improves execution success from 30.4\% to 53.1\%, outperforms DiffusionNFT by 10.2 percentage points, and yields rewards better aligned with human and simulator judgments.
% \end{itemize}

\section{Related Work}
\label{sec:related}

\textbf{Video Generation Model for Robot Manipulation.}
Video generative models (VGMs) have emerged as a promising foundation for robotic manipulation~\cite{unipi,ajay2023vlp,bharadhwaj2024gen2act,hu2024vpp,dream2flow,lvp,vidar,motus,dit4dit,dreamzero}.
Most methods follow a two-stage pipeline: generating a video plan from the task instruction and current observation, and then retargeting the predicted motion into robot actions via inverse dynamics models (IDMs).
Recent progress has improved both physically plausible video prediction~\cite{lvp,cosmos,chen2026abotphysworld} and general-purpose IDMs~\cite{vidar,tan2025anypos}.
However, a fundamental gap remains between generating globally coherent videos and producing robot actions that can be successfully executed in the physical world~\cite{vampo,wang2026eva}.
While EmboAlign~\cite{emboalign} uses compositional constraints to filter implausible videos before action retargeting, our work focuses on improving the VGM itself through constraint-guided post-training.

\textbf{Reward Modeling for Video Generation.}
Designing effective rewards for video generation is challenging because task-relevant failures may be sparse in space and time.
Existing video reward models either adapt image-level metrics such as CLIP and human preference scores~\cite{tagrpo,clip,hpsv3}, or learn preference rewards for visual quality, motion quality, and text alignment~\cite{liu2025videoalign}.
For physical applications, recent methods further evaluate specific properties such as 3D consistency~\cite{videogpa}, action smoothness and reachability~\cite{wang2026eva}, or visual distortions~\cite{chen2026abotphysworld}.
Nevertheless, these rewards remain insufficient for manipulation videos, where downstream success depends on compositional requirements and subtle robot-object interactions.
Inspired by constraint-based task specifications~\cite{emboalign,rekep,codeasmonitor,safebimanual}, we formulate manipulation success as LTL-style temporal compositions over continuous predicates, enabling faithful success evaluation and localized failure diagnosis.

\textbf{Reinforcement Learning and Preference Alignment for Video Generation.}
Reinforcement learning and preference alignment have become important tools for adapting generative models to objectives beyond likelihood-based pretraining.
Existing post-training methods for flow or diffusion VGMs largely inherit image-generation RL strategies, including policy-gradient training with stochastic samplers~\cite{dancegrpo,flowgrpo,mixgrpo2025} and reward-weighted denoising objectives such as Diffusion-NFT~\cite{diffusionnft}.
Recent video-oriented methods~\cite{tagrpo,sagegrpo} further show that directly applying image-level objectives to videos can underuse within-condition sample relationships and introduce unstable latent perturbations.
Manipulation video generation amplifies these challenges, since reliable dense rewards are unavailable, subtle pixel-level errors can determine downstream execution, and most pixels are irrelevant to the actual failure.
To address this, \methodname{} turns sparse-reward video diffusion RL into a localized corrective estimator: violation traces from the constraint monitor confine the reward-induced gradient to failure-relevant regions, and the corrective reflow loss replaces the symmetric reflection target used by negative-aware finetuning~\cite{diffusionnft} with a distributional anchor toward the within-condition positive marginal, estimated by the empirical mean of within-group successful rollouts.

\section{Problem Formulation and Preliminaries}
\label{sec:problem}

\paragraph{Video policy and objective.}
We study online reinforcement-learning post-training of a pretrained
flow-matching text-image-to-video (TI2V) model for robot manipulation.
Let $y=(c,\mathbf{I})\sim\mathcal{D}_{\rm task}$ denote a task context,
where $c$ is the language instruction and $\mathbf{I}$ is the first-frame
observation. Following the video-as-policy pipeline of Vidar~\cite{vidar},
a trainable video generator samples a rollout
$\mathbf{x}_0\sim p_\theta(\cdot\mid y)$, and a frozen inverse-dynamics
decoder $\mathcal{I}_{\rm IDM}$ maps the rollout to robot actions
$\mathbf{a}=\mathcal{I}_{\rm IDM}(\mathbf{x}_0)$. Only the video generator
parameters $\theta$ are updated during post-training.

Let $\mathrm{Succ}(\mathbf{a};y)\in\{0,1\}$ be the binary indicator of
successful execution in the target environment. The ideal objective is
\begin{equation}
    J_{\rm succ}(\theta)
    =
    \mathbb{E}_{y\sim\mathcal{D}_{\rm task}}\,
    \mathbb{E}_{\mathbf{x}_0\sim p_\theta(\cdot\mid y)}
    \left[
        \mathrm{Succ}\!\left(\mathcal{I}_{\rm IDM}(\mathbf{x}_0);y\right)
    \right].
    \label{eq:ideal_objective}
\end{equation}
Evaluating $\mathrm{Succ}$ requires executing decoded actions on a real robot
or high-fidelity simulator, so it is too expensive for the inner training
loop. During post-training, we instead score generated rollouts with the
compositional constraint introduced in Sec.~\ref{sec:reward}. 

\paragraph{DiffusionNFT surrogate.}
\methodname{} builds on the DiffusionNFT~\cite{diffusionnft} surrogate for online reinforcement learning of diffusion or flow models. In each online iteration, we sample rollouts from the behavior policy $p_{\theta_{\rm old}}(\cdot\mid y)$ and use $i$ to index samples in this rollout batch. For each $\mathbf{x}_0^{(i)}$, draw $t\sim\mathcal{U}[0,1]$ and $\boldsymbol{\epsilon}\sim\mathcal{N}(0,I)$, and use the rectified-flow interpolation
\begin{equation}
    \mathbf{x}_t^{(i)}
    =
    (1-t)\mathbf{x}_0^{(i)} + t\boldsymbol{\epsilon},
    \qquad
    \mathbf{v}^{(i)}
    =
    \boldsymbol{\epsilon}-\mathbf{x}_0^{(i)}.
    \label{eq:flow_interpolation}
\end{equation}
Let $v_\theta$ be the trainable velocity field and $v_{\rm old}=\operatorname{sg}(v_{\theta_{\rm old}})$ be the stop-gradient behavior velocity. DiffusionNFT defines positive and negative branches
\begin{equation}
    v_\theta^{+}
    =
    (1-\beta)v_{\rm old}+\beta v_\theta,
    \qquad
    v_\theta^{-}
    =
    (1+\beta)v_{\rm old}-\beta v_\theta,
    \label{eq:nft_branches}
\end{equation}
where $\beta>0$ controls the reinforcement strength. Its loss is
\begin{equation}
    \mathcal{L}_{\mathrm{NFT}}(\theta)
    =
    \mathbb{E}_{y,i,t,\boldsymbol{\epsilon}}
    \Bigl[
        r_i
        \,
        \bigl\|
            v_\theta^{+}(\mathbf{x}_t^{(i)},t,y)
            -
            \mathbf{v}^{(i)}
        \bigr\|_2^2
        +
        (1-r_i)
        \,
        \bigl\|
            v_\theta^{-}(\mathbf{x}_t^{(i)},t,y)
            -
            \mathbf{v}^{(i)}
        \bigr\|_2^2
    \Bigr],
    \label{eq:nft_loss}
\end{equation}
where $r_i\in\{0,1\}$ denotes the reward of rollout
$\mathbf{x}_0^{(i)}$. Intuitively, $\mathcal{L}_{\mathrm{NFT}}$ performs implicit preference learning: the $r_i=1$ branch reduces to a rejection-style~\cite{lee2023aligning} flow-matching update on successful rollouts, while the $r_i=0$ branch keeps $v_{\rm old}$ as an anchor and trains $v_\theta$ toward the reflected target of a failed rollout, thereby pushing the policy away from low-reward samples rather than simply discarding them.
%=========================================================
\section{Method}
\label{sec:method}
%=========================================================
This section presents \methodname{}, an online RL framework that converts compositional task violations into localized corrective updates for post-training manipulation VGMs; see Fig.~\ref{fig:method}. 
The key challenge is not only to obtain a faithful success signal, but also to make this sparse binary signal usable for optimizing high-dimensional video rollouts.  
\methodname{} first evaluates each rollout with a task-specific finite-trace LTL monitor (\S\ref{sec:reward}), which verifies a compositional task specification and returns both a binary success label and violation traces over constraints, entities, and frames. 
It then uses these traces to build two localized objectives: a group-shared credit-aware NFT loss (\S\ref{sec:mask}) that restricts negative-aware finetuning to reward-relevant spatio-temporal regions, and a corrective reflow loss (\S\ref{sec:aux}) that pulls failed rollouts toward the empirical mean of same-condition successful rollouts on their violated regions.

\begin{figure}[t]
  \centering
  \includegraphics[width=\linewidth]{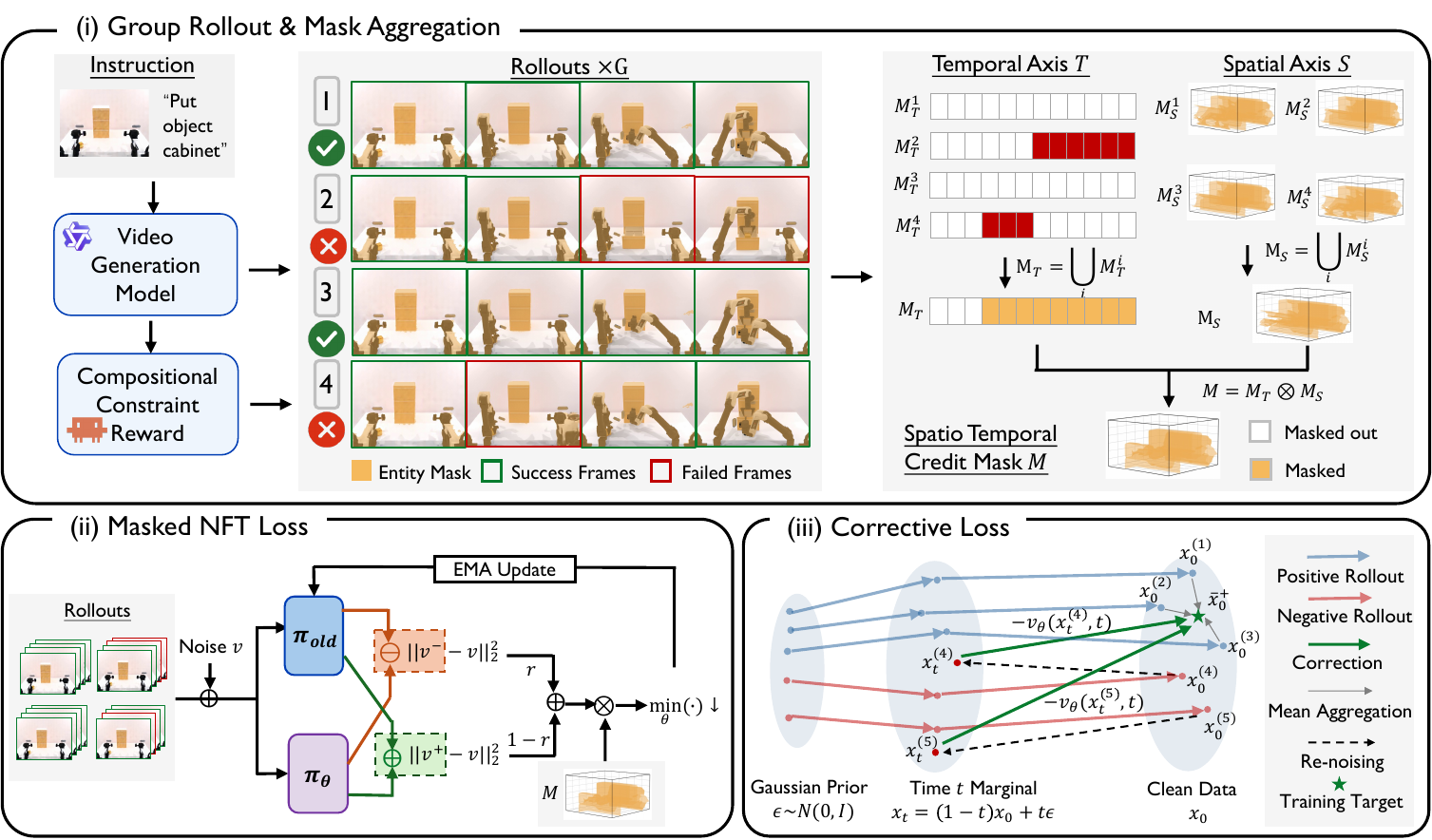}
  \caption{\textbf{Overview of \methodname{}.} 
    \methodname{} turns sparse compositional reward feedback into localized supervision for video diffusion post-training.
    Instead of applying reward-induced updates uniformly over the full rollout, it identifies the spatio-temporal mask responsible for task success or failure and restricts optimization to these regions.
    The resulting objective combines localized negative-aware finetuning with an explicit corrective target from same-condition successful rollouts, enabling stable improvement under sparse binary manipulation rewards.}
      \vspace{-3mm}
  \label{fig:method}
\end{figure}

\subsection{Compositional Constraint Monitor}
\label{sec:reward}
Existing video reward models for embodied generation either reduce a rollout to a single quality dimension~\cite{wang2026worldcompass,liu2025videoalign,tagrpo} or rely on a VLM-as-judge that subsamples a few frames and overlooks the small objects and short-window events on which manipulation success actually pivots~\cite{chen2026abotphysworld}. 
Manipulation success, by contrast, is compositional: several temporal and causal conditions must hold jointly throughout a rollout. Inspired by recent work that specifies embodied task goals and safety with Temporal Logic specification~\cite{sentinel2026,li2024embodied}, \methodname{} introduces a compositional constraint monitor $\mathcal{M}_c$ that lifts each rollout into a perception-grounded state trace, evaluates a conjunction of finite-trace Linear Temporal Logic clauses, and returns
\begin{equation}
    (r,\mathcal{V},\mathcal{S}) = \mathcal{M}_c(x_0;y),
    \label{eq:monitor}
\end{equation}
where $r\!\in\!\{0,1\}$ surrogates $\mathrm{Succ}$ in Eq.~\eqref{eq:ideal_objective}, $\mathcal{V}$ is the per-clause violation record, and $\mathcal{S}$ is the per-entity SAM3 mask atlas. Together $(\mathcal{V},\mathcal{S})$ encode \emph{when} and \emph{where} the rollout's behavior matters; how they are aggregated into a spatio-temporal supervision mask is deferred to \S\ref{sec:mask}, which exploits the group structure of our RL algorithm. We instantiate $\mathcal{M}_c$ in two stages: (i) offline specification synthesis which run once per task by a coding agent and frozen during RL; and (ii) per-rollout state lifting and clause evaluation.

\textbf{Specification synthesis.} Given the task context $y=(c,\mathbf{I})$, an offline coding agent synthesizes a triple $(\mathcal{E}_c,\mathcal{P}_c,\Phi_c)$: (i) a set of task-relevant entities $\mathcal{E}_c$ (task-related robot and objects); (ii) an entity-grounded predicate library $\mathcal{P}_c$ in which each predicate is already bound to concrete entities, e.g.\ $\textsc{Grasp}(\text{arm},\text{cup})$ or $\textsc{In}(\text{cup},\text{drawer})$; and (iii) an LTL$_f$ conjunction $\Phi_c=\bigwedge_{k=1}^{K}\varphi_k$ in which each clause $\varphi_k$ is instantiated from a \emph{closed template library} covering four families: \emph{persistence} $\mathbf{G}\,p$, \emph{terminal placement} $\mathbf{F}\mathbf{G}\,p$, \emph{causal coupling} $\mathbf{G}(p\!\rightarrow\!q)$, and \emph{ordering} $p\,\mathbf{U}\,q$. The closed template library makes the next stage's witness extraction template-driven rather than ad-hoc.

\textbf{State lifting and clause evaluation.} For each rollout $x_0$, we lift the video into a per-entity state trace $\sigma(x_0;y)=(s_1,\ldots,s_T)$ with $s_t=\{s_t^{e}\}_{e\in\mathcal{E}_c}$, where $s_t^{e}$ collects the attributes of entity $e$ at frame $t$ (segmentation mask, 2D position, gripper open/close, attribute flags). Three vision foundation models cooperate on this lifting: SAM3~\cite{sam3} for object masks and trajectories, IDM~\cite{vidar,tan2025anypos} for robot and gripper states, and a VLM~\cite{gemini} for attribute checks such as whether a drawer is open. As a by-product, we collect each entity's swept SAM3 mask $m_{e}=\bigcup_{t}\mathrm{SAM3}(e,t)$ across the trace into the atlas $\mathcal{S}=\{m_{e}\}_{e\in\mathcal{E}_c}$. Each predicate $p\!\in\!\mathcal{P}_c$ is then evaluated as a Boolean stream $b_p(t)=p(\sigma,t)$, and each clause $\varphi_k$ is evaluated by a finite-trace bottom-up monitor that returns a truth value $r_k$ together with a violation witness $W_k\!\subseteq\!\mathcal{E}_c\times[1,T]$ obtained by template-driven extraction (e.g., $\mathbf{G}(p\!\rightarrow\!q)$ produces $W_k=\{(e,t):b_p(t)\wedge\neg b_q(t)\}$, $\mathbf{G}\,p$ produces the unsatisfied frames of $b_p$); satisfied clauses contribute $W_k=\emptyset$. The rollout reward is the binary verdict $r(x_0;y)=\mathbf{1}[\sigma\models\Phi_c]\!\in\!\{0,1\}$; we pack the per-clause witnesses into $\mathcal{V}=\{(k,W_k)\}_{k=1}^{K}$, which together with $\mathcal{S}$ is consumed by the credit-aware NFT loss in \S\ref{sec:mask}.

\subsection{Credit-Aware NFT Loss}
\label{sec:mask}

Video diffusion models update a high-dimensional pixel--frame tensor whose data manifold is low-dimensional~\cite{pope2021intrinsic,rombach2022ldm} and whose reward depends on only a small spatio-temporal region of that tensor~\cite{vipo,localdpo,densedpo}. Online RL is therefore prone to accumulate gradient error on coordinates outside this region, eventually causing off-manifold drift and policy collapse~\cite{grpoguard2025,sagegrpo}. Our binary reward makes things worse: an all-or-nothing gain governs the entire rollout, so reward-irrelevant pixels are pulled toward the positive or reflection target at the same strength as those carrying signal. We therefore restrict the residual of Eq.~\eqref{eq:nft_loss} to a spatio-temporal mask $\mathbf{M}$, aggregated from the per-rollout outputs of \S\ref{sec:reward}.

\textbf{Spatial-temporal credit mask.} We use a single $\mathbf{M}$ shared across all $N$ rollouts of a group rather than $N$ per-sample masks, so that the positive and negative branches of the loss compete on a common support---successful rollouts are reinforced at exactly the coordinates where failed rollouts' drift is penalized. Let $\pi_{t}:\mathcal{E}_c\times[1,T]\to 2^{[1,T]}$ denote the projection of a witness set onto its frame indices. The temporal axis collects every frame implicated by any failing clause across the group, and the spatial axis unions the entity atlases of all rollouts:
\begin{equation}
\mathbf{M}_T \;=\; \bigcup_{i,k}\, \pi_{t}\!\big(W_k^{(i)}\big), \qquad
\mathbf{M}_S \;=\; \bigcup_{i,e}\, m^{(i)}_{e}, \qquad
\mathbf{M} \;=\; \mathbf{M}_T \otimes \mathbf{M}_S.
\label{eq:group_mask}
\end{equation}
We treat $\mathbf{M}_T\!\in\!\{0,1\}^{T}$ and $\mathbf{M}_S\!\in\!\{0,1\}^{H\times W}$ as binary masks identified with their support sets, so $\otimes$ produces a binary $\mathbf{M}\!\in\!\{0,1\}^{T\times H\times W}$. $\mathbf{M}_T$ contributions come only from failed clauses (satisfied ones give $W_k^{(i)}=\emptyset$), whereas $\mathbf{M}_S$ unions all rollouts regardless of $r^{(i)}$ to preserve the shared support. If the whole group succeeds, $\mathbf{M}_T=\emptyset$ and $\mathbf{M}\equiv 0$, matching the vanishing group-relative advantage.

\textbf{Credit-aware NFT loss.} Gating both branches of Eq.~\eqref{eq:nft_loss} by $\mathbf{M}$ yields
\begin{equation}
  \Lnft^{\mathrm{CA}}(\theta) \;=\; \mathbb{E}_{i,t,\epsilon}\!\left[\, r_i\big\| \mathbf{M}\odotm(v_\theta^{+} - v^{(i)})\big\|^{2} \,+\, (1-r_i)\big\| \mathbf{M}\odotm(v_\theta^{-} - v^{(i)})\big\|^{2}\,\right],
  \label{eq:masked_nft}
\end{equation}
where $\|\cdot\|^{2}$ is the element-wise sum of squares, $\odotm$ broadcasts along the channel dimension, and $\mathbf{M}\!\equiv\!1$ recovers Eq.~\eqref{eq:nft_loss} exactly. Analysis in Appendix~\ref{app:nft_noise} shows that our credit-aware NFT loss preserves the improvement direction of the original NFT loss (Eq.~\eqref{eq:nft_loss}) while restricting updates to task-related pixels, thereby avoiding error accumulation on unrelated pixels---a substantial source of off-manifold drift in negative-aware finetuning~\cite{grpoguard2025,sagegrpo}. This design is critical on high-dimensional video: estimating the improvement direction jointly from a small rollout batch and the velocity anchor $v_{\rm old}$ is substantially noisy, and since reward depends on only a small spatio-temporal region of the pixel volume~\cite{vipo,localdpo,densedpo}, our mask is essential to confine updates to that region and prevent the estimator noise from accumulating on the dominant task-irrelevant pixels.

\subsection{Corrective Reflow Loss}
\label{sec:aux}

Existing video RL post-training methods process each rollout in a group in isolation: every $\mathbf{x}_0^{(i)}$ contributes only through its own residual paired with its scalar reward $r_i$, leaving the relational structure across the group unused~\cite{diffusionnft,dancegrpo,flowgrpo,mixgrpo2025}. Our setting makes this structure unusually informative: under TI2V conditioning $y=(c,\mathbf{I})$, every within-group rollout starts from the same first frame, so the within-group successful rollouts are i.i.d.\ samples from the same condition-specific successful-rollout distribution, and their empirical mean is a pixel-level corrective continuation of any failed rollout exactly where the failure occurred. Instead of letting this relational signal act only implicitly, \methodname{} proposes a novel corrective reflow loss that explicitly supervises how each negative should be corrected on its failed regions, leading to faster and more stable post-training.

Let $\mathcal{P}=\{i:r_i\!=\!1\}$ and $\mathcal{N}=\{i:r_i\!=\!0\}$ denote the within-group positive and negative index sets. We use the empirical mean of the within-group positives as the corrective target:
\begin{equation}
  \xbarpos \;:=\; \frac{1}{|\mathcal{P}|}\sum_{j\in\mathcal{P}}\mathbf{x}_0^{(j)}.
  \label{eq:positive_mean}
\end{equation}
Conditioned on $y$, the positives in $\mathcal{P}$ are i.i.d.\ draws from the same successful-rollout distribution, so $\xbarpos$ is an unbiased estimator of the within-condition positive mean (App.~\ref{app:masks_quadratic}). Aggregating over the entire positive set also eliminates the matching-bias term that nearest-positive selection in any embedding would introduce.

\textbf{Loss form.} We reuse the group-shared mask $\mathbf{M}$ of Eq.~\eqref{eq:group_mask}: its spatial support $\mathbf{M}_S$ unions the entity atlases of both within-group positives and negatives, jointly bounding the failure-relevant support on the negatives and the corresponding corrective-content support on the positives---the dual support a substitution mask requires.  Inverting the rectified-flow interpolation Eq.~\eqref{eq:flow_interpolation} gives the model's one-step $\mathbf{x}_0$-prediction $\xhat_0^{(i)} := \mathbf{x}_t^{(i)} - t\,v_\theta(\mathbf{x}_t^{(i)},t,y)$. The corrective reflow loss regresses $\xhat_0^{(i)}$ toward $\xbarpos$ on $\mathbf{M}$, applied only to negatives:
\begin{equation}
  \Lcr(\theta) \;=\; \mathbb{E}_{i\in\mathcal{N},\,t,\,\epsilon}\!\left[\,\bigl\|\,\mathbf{M}\odotm \big(\xhat_0^{(i)} - \xbarpos\big)\,\bigr\|^{2}\,\right].
  \label{eq:aux}
\end{equation}
Analysis in Appendix~\ref{app:variance} shows that our corrective reflow loss stabilizes training with a strictly lower-variance \emph{parameter-gradient} estimator: the within-group positive mean directly supplies the reward-induced improvement direction, eliminating the EMA-plug-in noise floor that NFT~\cite{diffusionnft} incurs from reconstructing the same direction implicitly through the reflection of $v_{\rm old}$.

\textbf{Overall training objective.} Combining the credit-aware NFT loss, the corrective reflow loss, and a KL regularizer to a frozen reference policy $\pi_{\text{ref}}$ with parameters $\theta_{\text{ref}}$ (the pretrained Vidar policy used to initialize post-training), the full \methodname{} objective reads
\begin{equation}
  \Lcore(\theta) \;=\; \Lnft^{\mathrm{CA}}(\theta) \;+\; \lambda_{\rm CR}\Lcr(\theta) \;+\; \lambda_{\text{KL}}\,\mathcal{L}_{\text{KL}}(\theta; \theta_{\text{ref}}),
  \label{eq:total}
\end{equation}
where $\lambda_{\rm CR},\lambda_{\text{KL}}>0$ are scalar weights and $\mathcal{L}_{\text{KL}}$ is the standard velocity-MSE KL surrogate to the reference policy. 

%=========================================================
\section{Experiments}
\label{sec:exp}
%=========================================================
We design experiments to answer two questions: (Q1) does the compositional reward of \S\ref{sec:reward} better reflect downstream manipulation success than existing video reward models, and does post-training under it yield higher real-task success? (Q2) does \methodname{} optimize a sparse binary task reward more effectively than prior diffusion/flow RL algorithms? \S\ref{sec:exp_setup} fixes the protocol; \S\ref{sec:exp_main} answers Q1+Q2 jointly via downstream task success and isolates Q1 on a held-out reward-fidelity set; \S\ref{sec:exp_ablation} dissects \methodname{}'s components.

\subsection{Experimental Setup}
\label{sec:exp_setup}

\textbf{Tasks and base models.} We evaluate on eight RoboTwin~\cite{robotwin} bimanual manipulation tasks covering pick-and-place, articulated manipulation, multi-object arrangement, and goal-conditioned placement. Following the video-as-policy pipeline of Vidar~\cite{vidar}, we use the pretrained Vidar text-image-to-video model as the base VGM $G_\theta$ and Vidar's pretrained inverse dynamics model as the frozen action decoder $\mathcal{I}$; only $G_\theta$ is updated during post-training. The compositional constraint monitor of \S\ref{sec:reward} is synthesized once per task with Claude Opus 4.6 and frozen.

\textbf{Baselines.} We compare against three groups. (i) \emph{Policy-level reference}: $\pi_{0.5}$~\cite{pi05} and the un-post-trained Vidar base, reflecting the absolute scale of downstream success rather than competing with our post-training. (ii) \emph{RL post-training}: DanceGRPO~\cite{dancegrpo} and DiffusionNFT~\cite{diffusionnft}, applied to the same Vidar base under our binary reward with identical rollout budget and group size. (iii) \emph{Reward baselines} (Tab.~\ref{tab:reward}): EVA~\cite{wang2026eva} (IDM-based continuous), WMReward~\cite{chen2026abotphysworld} (JEPA-based continuous), VideoAlign~\cite{liu2025videoalign} (appearance-based continuous), and a VLM-as-judge that prompts a frozen VLM (Gemini 3 Flash~\cite{gemini}) for a binary verdict on downsampled frames. The same Gemini 3 Flash checkpoint is used for the attribute-state predicate $\textsc{State}(e,q)$ inside our monitor, so reward-fidelity comparisons against the VLM-as-judge baseline isolate the effect of LTL composition rather than VLM capability.

\textbf{Metrics.} For Q1+Q2 (\S\ref{sec:exp_main}, \S\ref{sec:exp_ablation}), we report \emph{task success rate}: the fraction of episodes in which $\mathcal{I}(\mathbf{x}_0)$, executed in RoboTwin, completes the task; each cell averages 100 episodes per task. For Q1 reward fidelity (Tab.~\ref{tab:reward}), we report agreement on a held-out 100-video set with two independent ground-truth labels (binary human and binary simulator-execution); see Tab.~\ref{tab:reward} caption for per-column definitions.

% Implementation details moved to Appendix~\ref{app:impl}.

\subsection{Main Results}
\label{sec:exp_main}

\textbf{Benchmark comparison.} Tab.~\ref{tab:main} compares \methodname{} with policy-level references ($\pi_{0.5}$, Vidar base) and three video RL post-training baselines (EVA, DanceGRPO, DiffusionNFT) on all eight RoboTwin tasks. DanceGRPO, DiffusionNFT, and \methodname{} share our binary compositional reward and differ only in optimizer; EVA bundles its own continuous IDM-based reward. We see that: 
i) \methodname{} achieves the best success rate on every task, outperforming both VLA baseline ($\pi_{0.5}$), video policy baseline (Vidar) and all video-RL post-training baselines, showing that our compositional reward and localized optimization strategy translate into consistent downstream execution gains. 
ii) EVA, the closest existing manipulation-video-RL baseline, improves only marginally over the Vidar base on average, suggesting that IDM-derived trajectory regularization alone is insufficient for evaluating contact-rich manipulation success. 
This is expected: penalizing low-level motion artifacts cannot verify whether a rollout satisfies the temporal and causal requirements of the task. 
iii) Under the same binary compositional reward, \methodname{} substantially outperforms DanceGRPO and vanilla DiffusionNFT, indicating that the proposed credit-aware NFT and corrective reflow objectives are more effective for optimizing sparse video rewards than generic group-relative policy gradients or unlocalized negative-aware finetuning. 
iv) The gains are especially large on long-horizon, multi-stage tasks such as \emph{put\_bottles\_dustbin} and \emph{stack\_bowls\_three}, where success depends on satisfying a sequence of compositional constraints. 
This supports our design hypothesis: LTL-based verification provides faithful judgments for long-horizon rollouts, while violation traces allow \methodname{} to focus updates on the reward-critical spatio-temporal regions.

\begin{table}[t]
  \centering
  \footnotesize
  \setlength{\tabcolsep}{4pt}
  \caption{Downstream task success rate on eight RoboTwin manipulation tasks. Top block: policy-level references reflect absolute scale. Bottom block: video RL post-training methods, all built on the same Vidar base with the same rollout budget; DanceGRPO, DiffusionNFT, and \methodname{} use our binary compositional reward, while EVA is a recent video-RL-for-robotic-manipulation method that uses its own continuous IDM-based reward.}
  \label{tab:main}
  \resizebox{\linewidth}{!}{%
  \begin{tabular}{lcccccccc|c}
    \toprule
    Method & rank-blk & put-bot & put-cab & stk-bowl & stk-blk & dump-bin & open-lap & hand-mic & Avg. \\
    \midrule
    $\pi_{0.5}$~\cite{pi05}                             & 25\% & 11\% & 33\% & 37\% &  9\% & 49\% & 46\% & 31\% & 30.1\% \\
    Vidar~\cite{vidar}                              & 52\% &  3\% & 22\% & 39\% & 25\% & 42\% & 36\% & 24\% & 30.4\% \\
    \midrule
    EVA~\cite{wang2026eva}                              & 45\% &  5\% & 20\% & 44\% & 39\% & 56\% & 50\% &  15\% & 34.3\% \\
    DanceGRPO~\cite{dancegrpo}                              & 48\% &  5\% & 20\% & 38\% & 22\% & 41\% & 35\% & 26\% & 29.4\% \\
    DiffusionNFT~\cite{diffusionnft}                                        & 52\% & 14\% & 36\% & 58\% & 41\% & 53\% & 49\% & 36\% & 42.4\% \\
    \textbf{\methodname{} (Ours)}                       & \textbf{61\%} & \textbf{24\%} & \textbf{49\%} & \textbf{75\%} & \textbf{54\%} & \textbf{61\%} & \textbf{57\%} & \textbf{44\%} & \textbf{53.1\%} \\
    \bottomrule
  \end{tabular}%
  }
\end{table}

\textbf{Training convergence.}
Fig.~\ref{fig:convergence} compares the post-training dynamics of DanceGRPO, DiffusionNFT, and \methodname{} on three representative RoboTwin tasks: \emph{rank\_blocks\_rgb}, \emph{put\_bottles\_dustbin}, and \emph{stack\_bowls\_three}.
All methods start from the same Vidar base and optimize the same binary compositional reward, so the curves directly reflect how effectively each optimizer converts sparse task feedback into model improvement.
We see that:
i) \methodname{} consistently rises faster and reaches the highest accumulated reward on all three tasks, indicating that localized correction improves not only final performance but also optimization efficiency.
ii) DanceGRPO shows unstable or limited improvement: its SDE-style exploration and rollout-level policy-gradient update are not targeted to the sparse frames and entities that determine manipulation success, making the relative-advantage signal weak under binary rewards.
iii) DiffusionNFT is more stable than DanceGRPO, but it plateaus at a lower reward because its negative-aware update is still applied over the full high-dimensional rollout without knowing which frames or entities caused failure.

\begin{figure}[t]
  \centering
  \includegraphics[width=0.98\linewidth]{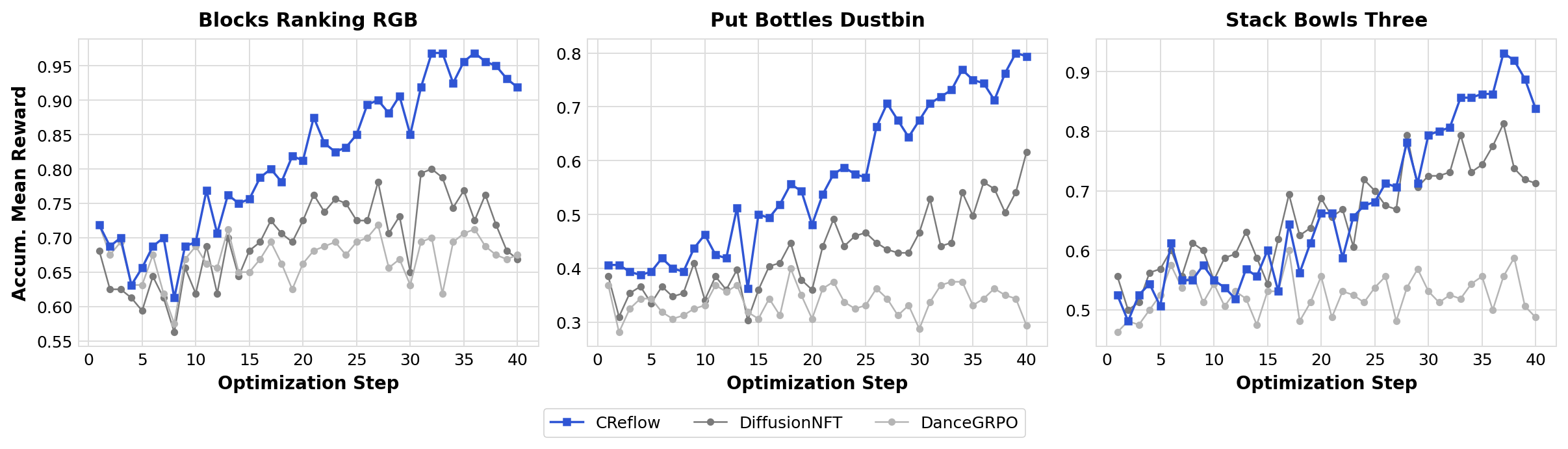}
  \caption{Training convergence on three RoboTwin tasks under the same binary compositional reward.
  Compared with DanceGRPO and DiffusionNFT, \methodname{} converts sparse task feedback into faster and more stable reward improvement by localizing updates to violation-relevant frames and entities.}
  \label{fig:convergence}
\end{figure}

\textbf{Reward fidelity to downstream success.}
Tab.~\ref{tab:reward} compares our compositional reward with continuous video reward models and a VLM-as-judge baseline on 100 held-out videos, each labeled by both human judgment and simulator execution.
We see that:
i) \methodname{} best matches downstream success, achieving $88.0\%$ agreement with human labels and $91.0\%$ with simulator labels.
ii) Its pairwise ranking accuracy also reaches $0.893$, clearly above the strongest continuous baseline EVA.
iii) This gap suggests that global appearance, motion, or executability scores are insufficient proxies for manipulation success, which requires verifying a compositional task specification.
iv) The VLM-as-judge baseline further shows that binary scoring alone is not enough; the advantage comes from LTL-based temporal verification over perception-grounded predicates.

\begin{table}[t]
  \centering
  \footnotesize
  \setlength{\tabcolsep}{6pt}
  \caption{Reward fidelity on 100 held-out evaluation videos with two independent ground-truth labels (human binary; simulator execution binary). Continuous rewards are reported with AUROC; binary rewards collapse AUROC to a single ROC point (marked ``---''). Pairwise rank acc.\ is computed on all (success, failure) pairs under simulator labels.}
  \label{tab:reward}
  \resizebox{\linewidth}{!}{%
  \begin{tabular}{lcccccc}
    \toprule
    Reward & Type & Acc.\ vs Human & F1 vs Human & Acc.\ vs Sim & AUROC vs Sim & Pairwise Rank Acc. \\
    \midrule
    EVA~\cite{wang2026eva}                                & Cont.\ (IDM)        & 74.0\% & 73.2\% & 78.0\% & 0.815 & 0.792  \\
    WMReward~\cite{chen2026abotphysworld}                 & Cont.\ (JEPA)       & 70.0\% & 69.1\% & 73.0\% & 0.758 & 0.741  \\
    VideoAlign~\cite{liu2025videoalign}                   & Cont.\ (Appearance) & 63.0\% & 62.4\% & 65.0\% & 0.652 & 0.621 \\
    VLM-as-judge                                          & Binary              & 71.0\% & 70.3\% & 74.0\% & --- & 0.718  \\
    \textbf{\methodname{} (Ours)}                         & \textbf{Binary}     & \textbf{88.0\%} & \textbf{87.4\%} & \textbf{91.0\%} & --- & \textbf{0.893 } \\
    \bottomrule
  \end{tabular}%
  }
\end{table}

\textbf{Computational overhead of the LTL monitor.}
We profile the per-rollout cost of our reward pipeline on a single H100.
A single rollout takes $\sim\!60$\,s to sample, and one optimizer step on the LoRA-adapted backbone takes $\sim\!70$\,s---both costs that any video-diffusion-RL method must already pay.
On top of this, evaluating our compositional LTL reward on a generated video adds $\sim\!15$\,s in total: SAM3 mask extraction  is the dominant term ($\sim\!8$\,s), the IDM gripper trace contributes $\sim\!1$\,s, and the remainder is spent on LTL evaluation and VLM queries. Our reward computation therefore accounts for only a small fraction of the per-step wall-clock budget relative to rollout sampling and gradient update.

% \textbf{Compositional reward localizes failure modes.} Fig.~\ref{fig:ltl_examples} visualizes violation traces on four characteristic failure modes: object teleportation (\emph{rank\_blocks\_rgb}), wrong contact ordering (\emph{put\_bottles\_dustbin}), goal predicate miss (\emph{stack\_bowls\_three}), and articulated-state miss (\emph{put\_object\_cabinet}). We see that: i) each failure mode maps to a distinct LTL clause family from \S\ref{sec:reward} (kinematic, contact-ordering, goal, articulated-state), demonstrating that the agent-synthesized formula matches real failure-mode diversity; ii) in three of four cases, $\mathcal{F}_k$ localizes to a few contact-event frames, justifying the per-frame mask construction in \S\ref{sec:mask} over a video-level reward; iii) the violating-entity set $\mathcal{E}_k$ correctly excludes irrelevant entities (e.g., non-target bottles, the un-grasped arm), so the credit-aware NFT loss leaves non-violating regions anchored to the prior.

% \begin{figure}[t]
%   \centering
%   \includegraphics[width=0.98\linewidth]{figures/ltl_examples.pdf}
%   \caption{Compositional monitor traces on four failure modes. Each row: generated frame strip with violating entity highlighted (SAM3 mask), violation interval $\mathcal{F}_k$ underneath, failed LTL clause $\varphi_k$ printed below.}
%   \label{fig:ltl_examples}
% \end{figure}

\subsection{Ablation Study}
\label{sec:exp_ablation}

\textbf{Component ablation.} Tab.~\ref{tab:component} dissects \methodname{} into its two core mechanisms (the credit-aware NFT loss and the corrective reflow loss) on three representative tasks. We see that: i) credit-aware NFT alone ($43.0\%$ vs.\ $36.0\%$) closes most of vanilla NFT's credit-assignment gap, confirming that the binary verdict depends on a small task-relevant subset of pixels and frames; ii) corrective reflow alone ($39.0\%$) yields a smaller gain---without the violation-trace mask, the within-group positive prototype still drives gradient on off-task coordinates and dilutes the signal.

\begin{wraptable}{r}{0.50\linewidth}
  \centering
  \footnotesize
  \vspace{-5mm}
  \caption{Component ablation. Success rate (\%) on three representative RoboTwin tasks.}
  \label{tab:component}
  \setlength{\tabcolsep}{2.5pt}
  \resizebox{\linewidth}{!}{%
  \begin{tabular}{lccc|c}
    \toprule
    Method & put-bot & stk-bowl & put-cab & Avg.\ \\
    \midrule
    Vanilla DiffusionNFT                                & 14\% & 58\% & 36\% & 36.0\% \\
    \quad + credit-aware NFT only                       & 19\% & 67\% & 43\% & 43.0\% \\
    \quad + corrective reflow only                      & 16\% & 62\% & 39\% & 39.0\% \\
    \quad + \textbf{both (Ours)}                                      & \textbf{24\%} & \textbf{75\%} & \textbf{49\%} & \textbf{49.3\%} \\
    \bottomrule
  \end{tabular}%
  }
  \vspace{-3mm}
\end{wraptable}

\textbf{Qualitative ablation.} Fig.~\ref{fig:mask} is the visual counterpart of the credit-aware-NFT-vs-vanilla row in Tab.~\ref{tab:component}: vanilla DiffusionNFT's uniform contrastive update visibly degrades the dustbin region during RL training (Fig.~\ref{fig:mask_nomask}), even though those off-task pixels never enter the binary reward; restricting the update to $\mathbf{M}$ (Fig.~\ref{fig:mask_bottle}) anchors the rest of the scene to the pretrained prior and prevents this off-task collapse (Fig.~\ref{fig:mask_overlay}).

\begin{figure}[t]
  \centering
  \begin{subfigure}[b]{0.32\linewidth}
    \centering
    \includegraphics[width=\linewidth]{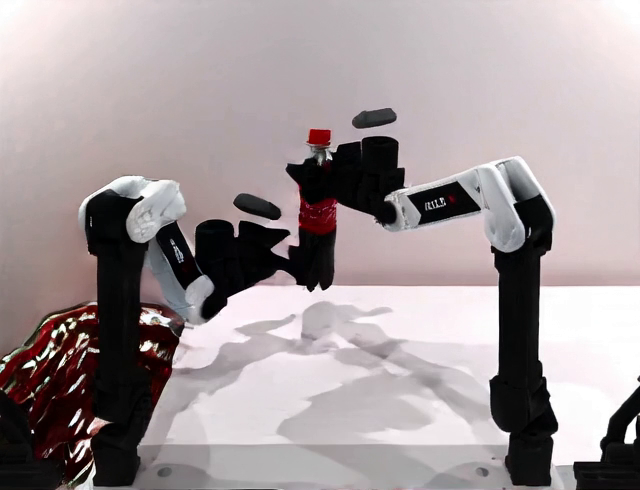}
    \caption{Vanilla DiffusionNFT (no mask).}
    \label{fig:mask_nomask}
  \end{subfigure}
  \hfill
  \begin{subfigure}[b]{0.32\linewidth}
    \centering
    \includegraphics[width=\linewidth]{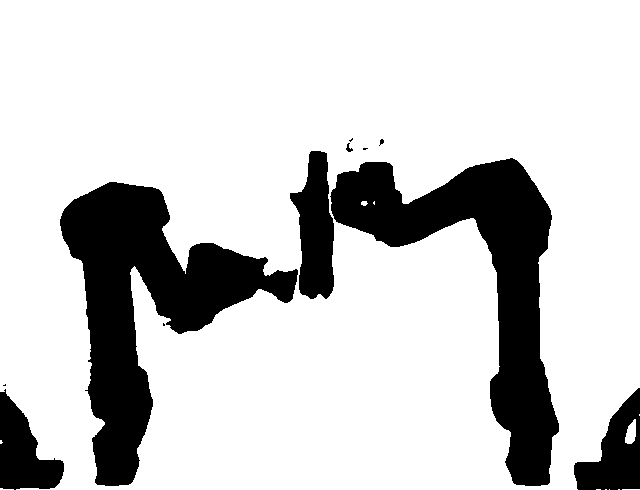}
    \caption{Group-shared mask $\mathbf{M}$.}
    \label{fig:mask_bottle}
  \end{subfigure}
  \hfill
  \begin{subfigure}[b]{0.32\linewidth}
    \centering
    \includegraphics[width=\linewidth]{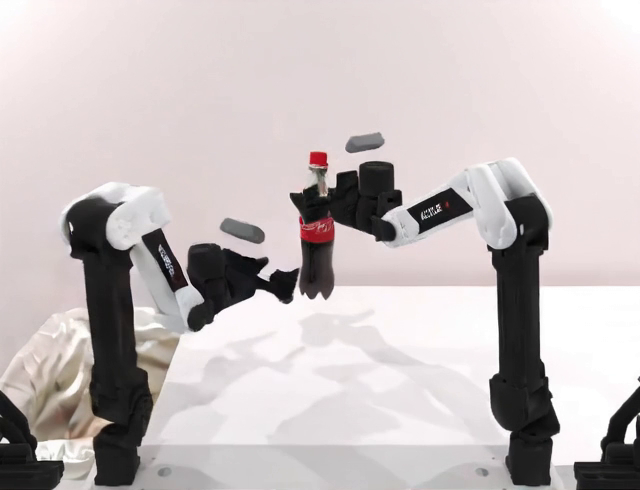}
    \caption{\methodname{} (Ours).}
    \label{fig:mask_overlay}
  \end{subfigure}
  \caption{Qualitative ablation on a \emph{put\_bottles\_dustbin} rollout. \textbf{(a)} Vanilla DiffusionNFT, with no mask, visibly degrades the dustbin region during RL training. \textbf{(b)} Our group-shared mask $\mathbf{M}$ covers the task-relevant entities (arms, bottle); its complement is anchored to the pretrained prior. \textbf{(c)} \methodname{} confines the contrastive update to $\mathbf{M}$, preserving off-task visual quality.}
  \label{fig:mask}
  \vspace{-3mm}
\end{figure}

%=========================================================
\section{Conclusion and Limitations}
\label{sec:conclusion}
%=========================================================
We propose \methodname{}, a corrective reflow framework for post-training manipulation VGMs with sparse compositional rewards. 
\methodname{} formulates task success as compositional LTL constraints and uses the resulting localized constraint violations to construct two targeted objectives: a credit-aware NFT loss that confines reward-induced updates to task-relevant regions, and a corrective reflow loss that uses same-condition successful rollouts as explicit correction targets for failed samples. 
Experiments show that \methodname{} yields reward judgments better aligned with human and simulator success labels than existing methods, leads to effective training to improve high dimensional video generation, and improves downstream execution success by 23.8 percentage points. \textbf{Limitation and future work.}
Our current monitor reasons over 2D image-plane states and therefore cannot fully resolve fine-grained 3D spatial relations.
Extending the constraint evaluation with depth or 3D state estimation is a natural direction for future work.

%=========================================================
% \begin{ack}
% % TODO: funding and competing interests disclosure (final version only).
% \end{ack}

\bibliographystyle{plainnat}
\bibliography{references}

@article{cosmos,
  title  = {Cosmos World Foundation Model Platform for Physical {AI}},
  author = {Agarwal, Niket and Ali, Arslan and Bala, Maciej and Balaji, Yogesh and others},
  journal= {arXiv preprint arXiv:2501.03575},
  year   = {2025}
}

@inproceedings{unipi,
  title    = {Learning Universal Policies via Text-Guided Video Generation},
  author   = {Du, Yilun and Yang, Mengjiao and Dai, Bo and Dai, Hanjun and Nachum, Ofir and Tenenbaum, Joshua B. and Schuurmans, Dale and Abbeel, Pieter},
  booktitle= {Advances in Neural Information Processing Systems (NeurIPS)},
  year     = {2023}
}

@misc{ajay2023vlp,
  title  = {Video Language Planning},
  author = {Du, Yilun and Yang, Mengjiao and Florence, Pete and Xia, Fei and Wahid, Ayzaan and Ichter, Brian and Sermanet, Pierre and Yu, Tianhe and Abbeel, Pieter and Tenenbaum, Joshua B. and Kaelbling, Leslie and Zeng, Andy and Tompson, Jonathan},
  year   = {2023},
  note   = {arXiv preprint arXiv:2310.10625}
}

@article{bharadhwaj2024gen2act,
  title={Gen2act: Human video generation in novel scenarios enables generalizable robot manipulation},
  author={Bharadhwaj, Homanga and Dwibedi, Debidatta and Gupta, Abhinav and Tulsiani, Shubham and Doersch, Carl and Xiao, Ted and Shah, Dhruv and Xia, Fei and Sadigh, Dorsa and Kirmani, Sean},
  journal={arXiv preprint arXiv:2409.16283},
  year={2024}
}

@misc{hu2024vpp,
  title  = {Video Prediction Policy: A Generalist Robot Policy with Predictive Visual Representations},
  author = {Hu, Yucheng and Guo, Yanjiang and Wang, Pengchao and Chen, Xiaoyu and Wang, Yen-Jen and Zhang, Jianke and Sreenath, Koushil and Lu, Chaochao and Chen, Jianyu},
  year   = {2024},
  note   = {arXiv preprint arXiv:2412.14803}
}

@article{lvp,
  title  = {Large Video Planner Enables Generalizable Robot Control},
  author = {Chen, Boyuan and Zhang, Tianyuan and Geng, Haoran and Song, Kiwhan and Zhang, Caiyi and Li, Peihao and Freeman, William T. and Malik, Jitendra and Abbeel, Pieter and Tedrake, Russ and Sitzmann, Vincent and Du, Yilun},
  journal= {arXiv preprint arXiv:2512.15840},
  year   = {2025}
}

@article{vidar,
  title  = {{Vidar}: Embodied Video Diffusion Model for Generalist Manipulation},
  author = {Feng, Yao and Tan, Hengkai and Mao, Xinyi and Xiang, Chendong and Liu, Guodong and Huang, Shuhe and Su, Hang and Zhu, Jun},
  journal= {arXiv preprint arXiv:2507.12898},
  year   = {2025}
}

@misc{tan2025anypos,
  title  = {{AnyPos}: Automated Task-Agnostic Actions for Bimanual Manipulation},
  author = {Tan, Hengkai and Feng, Yao and Mao, Xinyi and Huang, Shuhe and Liu, Guodong and Hao, Zhongkai and Su, Hang and Zhu, Jun},
  year   = {2025},
  note   = {arXiv preprint arXiv:2507.12768}
}

@article{dream2flow,
  title  = {{Dream2Flow}: Bridging Video Generation and Open-World Manipulation with {3D} Object Flow},
  author = {Dharmarajan, Karthik and Huang, Wenlong and Wu, Jiajun and Fei-Fei, Li and Zhang, Ruohan},
  journal= {arXiv preprint arXiv:2512.24766},
  year   = {2025}
}

@article{emboalign,
  title  = {{EmboAlign}: Aligning Video Generation with Compositional Constraints for Zero-Shot Manipulation},
  author = {Zhang, Gehao and Ni, Zhenyang and Mohapatra, Payal and Liu, Han and Zhang, Ruohan and Zhu, Qi},
  journal= {arXiv preprint arXiv:2603.05757},
  year   = {2026}
}

@misc{diffusionnft,
  title  = {{DiffusionNFT}: Online Diffusion Reinforcement with Forward Process},
  author = {Zheng, Kaiwen and Chen, Huayu and Ye, Haotian and Wang, Haoxiang and Zhang, Qinsheng and Jiang, Kai and Su, Hang and Ermon, Stefano and Zhu, Jun and Liu, Ming-Yu},
  year   = {2025},
  note   = {arXiv preprint arXiv:2509.16117}
}

@misc{flowgrpo,
  title  = {{Flow-GRPO}: Training Flow Matching Models via Online {RL}},
  author = {Liu, Jie and Liu, Gongye and Liang, Jiajun and Li, Yangguang and Liu, Jiaheng and Wang, Xintao and Wan, Pengfei and Zhang, Di and Ouyang, Wanli},
  year   = {2025},
  note   = {arXiv preprint arXiv:2505.05470}
}

@misc{dancegrpo,
  title  = {{DanceGRPO}: Unleashing {GRPO} on Visual Generation},
  author = {Xue, Zeyue and Wu, Jie and Gao, Yu and Kong, Fangyuan and Zhu, Lingting and Chen, Mengzhao and Liu, Zhiheng and Liu, Wei and Guo, Qiushan and Huang, Weilin and Luo, Ping},
  year   = {2025},
  note   = {arXiv preprint arXiv:2505.07818}
}

@misc{mixgrpo2025,
  title  = {{MixGRPO}: Unlocking Flow-based {GRPO} Efficiency with Mixed {ODE-SDE}},
  author = {Li, Junzhe and Cui, Yutao and Huang, Tao and Ma, Yinping and Fan, Chun and Yang, Miles and Zhong, Zhao and Bo, Liefeng},
  year   = {2025},
  note   = {arXiv preprint arXiv:2507.21802}
}

@misc{grpoguard2025,
  title  = {{GRPO-Guard}: Mitigating Implicit Over-Optimization in Flow Matching via Regulated Clipping},
  author = {Wang, Jing and Liang, Jiajun and Liu, Jie and Liu, Henglin and Liu, Gongye and Zheng, Jun and Pang, Wanyuan and Ma, Ao and Xie, Zhenyu and Wang, Xintao and Wang, Meng and Wan, Pengfei and Liang, Xiaodan},
  year   = {2025},
  note   = {arXiv preprint arXiv:2510.22319}
}

@inproceedings{pope2021intrinsic,
  title     = {The Intrinsic Dimension of Images and Its Impact on Learning},
  author    = {Pope, Phillip and Zhu, Chen and Abdelkader, Ahmed and Goldblum, Micah and Goldstein, Tom},
  booktitle = {International Conference on Learning Representations (ICLR)},
  year      = {2021},
  note      = {arXiv preprint arXiv:2104.08894}
}

@inproceedings{rombach2022ldm,
  title     = {High-Resolution Image Synthesis with Latent Diffusion Models},
  author    = {Rombach, Robin and Blattmann, Andreas and Lorenz, Dominik and Esser, Patrick and Ommer, Bj\"orn},
  booktitle = {IEEE/CVF Conference on Computer Vision and Pattern Recognition (CVPR)},
  year      = {2022},
  note      = {arXiv preprint arXiv:2112.10752}
}

@misc{tagrpo,
  title  = {{TAGRPO}: Boosting {GRPO} on Image-to-Video Generation with Direct Trajectory Alignment},
  author = {Wang, Jin and Lu, Jianxiang and Xu, Guangzheng and Chen, Comi and Yang, Haoyu and Wang, Linqing and Chen, Peng and Chen, Mingtao and Hu, Zhichao and Wu, Longhuang and others},
  year   = {2026},
  note   = {arXiv preprint arXiv:2601.05729}
}

@misc{vampo,
  title  = {{VAMPO}: Policy Optimization for Improving Visual Dynamics in Video Action Models},
  author = {Ge, Zirui and Ding, Pengxiang and Yin, Baohua and Wang, Qishen and Xie, Zhiyong and Wang, Yemin and Wang, Jinbo and Li, Hengtao and Suo, Runze and Song, Wenxuan and others},
  year   = {2026},
  note   = {arXiv preprint arXiv:2603.19370}
}

@misc{vipo,
  title  = {Seeing What Matters: Visual Preference Policy Optimization for Visual Generation},
  author = {Ni, Ziqi and others},
  year   = {2025},
  note   = {arXiv preprint arXiv:2511.18719}
}

@misc{localdpo,
  title  = {Mind the Generative Details: Direct Localized Detail Preference Optimization for Video Diffusion Models},
  author = {Huang, Zitong and Zhang, Kaidong and Ding, Yukang and Gao, Chao and Ding, Rui and Chen, Ying and Zuo, Wangmeng},
  year   = {2026},
  note   = {arXiv preprint arXiv:2601.04068}
}

@misc{densedpo,
  title  = {{DenseDPO}: Fine-Grained Temporal Preference Optimization for Video Diffusion Models},
  author = {Wu, Ziyi and Kag, Anil and Skorokhodov, Ivan and Menapace, Willi and Mirzaei, Ashkan and Gilitschenski, Igor and Tulyakov, Sergey and Siarohin, Aliaksandr},
  year   = {2025},
  note   = {arXiv preprint arXiv:2506.03517}
}

@misc{wang2026eva,
  title  = {{EVA}: Aligning Video World Models with Executable Robot Actions via Inverse Dynamics Rewards},
  author = {Wang, Ruixiang and Liu, Qingming and Deng, Yueci and Liu, Guiliang and Liu, Zhen and Jia, Kui},
  year   = {2026},
  note   = {arXiv preprint arXiv:2603.17808}
}

@misc{liu2025videoalign,
  title  = {Improving Video Generation with Human Feedback},
  author = {Liu, Jie and Liu, Gongye and Liang, Jiajun and Yuan, Ziyang and Liu, Xiaokun and Zheng, Mingwu and Wu, Xiele and Wang, Qiulin and Xia, Menghan and Wang, Xintao and others},
  year   = {2025},
  note   = {arXiv preprint arXiv:2501.13918}
}

@misc{sam3,
  title  = {{SAM 3}: Segment Anything with Concepts},
  author = {Carion, Nicolas and Gustafson, Laura and Hu, Yuan-Ting and Debnath, Shoubhik and Hu, Ronghang and Suris, Didac and Ryali, Chaitanya and others},
  year   = {2025},
  note   = {arXiv preprint arXiv:2511.16719. Meta AI}
}

@misc{wang2026worldcompass,
  title  = {{WorldCompass}: Reinforcement Learning for Long-Horizon World Models},
  author = {Wang, Zehan and Wang, Tengfei and Zhang, Haiyu and Zuo, Xuhui and Wu, Junta and Wang, Haoyuan and Sun, Wenqiang and Wang, Zhenwei and Cao, Chenjie and Zhao, Hengshuang and others},
  year   = {2026},
  note   = {arXiv preprint arXiv:2602.09022}
}

@misc{sagegrpo,
  title  = {Manifold-Aware Exploration for Reinforcement Learning in Video Generation},
  author = {Zheng, Mingzhe and Kong, Weijie and Wu, Yue and Jiang, Dengyang and Ma, Yue and He, Xuanhua and Lin, Bin and Gong, Kaixiong and Zhong, Zhao and Bo, Liefeng and others},
  year   = {2026},
  note   = {arXiv preprint arXiv:2603.21872}
}

@article{videogpa,
  title  = {{VideoGPA}: Distilling Geometry Priors for {3D}-Consistent Video Generation},
  author = {Du, Hongyang and Ye, Junjie and Cong, Xiaoyan and Li, Runhao and Ni, Jingcheng and Agarwal, Aman and Zhou, Zeqi and Li, Zekun and Balestriero, Randall and Wang, Yue},
  journal= {arXiv preprint arXiv:2601.23286},
  year   = {2026}
}

@article{chen2026abotphysworld,
  title  = {{ABot-PhysWorld}: Interactive World Foundation Model for Robotic Manipulation with Physics Alignment},
  author = {Chen, Yuzhi and Chen, Ronghan and Huo, Dongjie and Yang, Yandan and Qi, Dekang and Liu, Haoyun and Lin, Tong and Zeng, Shuang and Xiao, Junjin and Chang, Xinyuan and others},
  journal= {arXiv preprint arXiv:2603.23376},
  year   = {2026}
}

@article{robotwin,
  title  = {{RoboTwin 2.0}: A Scalable Data Generator and Benchmark with Strong Domain Randomization for Robust Bimanual Robotic Manipulation},
  author = {Chen, Tianxing and Chen, Zanxin and Chen, Baijun and Cai, Zijian and Liu, Yibin and Li, Zixuan and Liang, Qiwei and Lin, Xianliang and Ge, Yiheng and Gu, Zhenyu and others},
  journal= {arXiv preprint arXiv:2506.18088},
  year   = {2025}
}

@article{sentinel2026,
  title  = {{SENTINEL}: A Multi-Level Formal Framework for Safety Evaluation of Foundation Model-based Embodied Agents},
  author = {Zhan, Simon Sinong and Liu, Yao and Wang, Philip and Wang, Zinan and Wang, Qineng and Peng, Yiyan and Ruan, Zhian and Shi, Xiangyu and Cao, Xinyu and Yang, Frank and Wang, Kangrui and Shao, Huajie and Li, Manling and Zhu, Qi},
  journal= {arXiv preprint arXiv:2510.12985},
  year   = {2025}
}

@article{li2024embodied,
  title={Embodied agent interface: Benchmarking llms for embodied decision making},
  author={Li, Manling and Zhao, Shiyu and Wang, Qineng and Wang, Kangrui and Zhou, Yu and Srivastava, Sanjana and Gokmen, Cem and Lee, Tony and Li, Li E and Zhang, Ruohan and others},
  journal={Advances in Neural Information Processing Systems},
  volume={37},
  pages={100428--100534},
  year={2024}
}

@article{motus,
  title={Motus: A unified latent action world model},
  author={Bi, Hongzhe and Tan, Hengkai and Xie, Shenghao and Wang, Zeyuan and Huang, Shuhe and Liu, Haitian and Zhao, Ruowen and Feng, Yao and Xiang, Chendong and Rong, Yinze and others},
  journal={arXiv preprint arXiv:2512.13030},
  year={2025}
}

@article{dit4dit,
  title={DiT4DiT: Jointly Modeling Video Dynamics and Actions for Generalizable Robot Control},
  author={Ma, Teli and Zheng, Jia and Wang, Zifan and Jiang, Chunli and Cui, Andy and Liang, Junwei and Yang, Shuo},
  journal={arXiv preprint arXiv:2603.10448},
  year={2026}
}

@article{dreamzero,
  title={World action models are zero-shot policies},
  author={Ye, Seonghyeon and Ge, Yunhao and Zheng, Kaiyuan and Gao, Shenyuan and Yu, Sihyun and Kurian, George and Indupuru, Suneel and Tan, You Liang and Zhu, Chuning and Xiang, Jiannan and others},
  journal={arXiv preprint arXiv:2602.15922},
  year={2026}
}

@inproceedings{hpsv3,
  title={Hpsv3: Towards wide-spectrum human preference score},
  author={Ma, Yuhang and Wu, Xiaoshi and Sun, Keqiang and Li, Hongsheng},
  booktitle={Proceedings of the IEEE/CVF International Conference on Computer Vision},
  pages={15086--15095},
  year={2025}
}

@inproceedings{clip,
  title={Learning transferable visual models from natural language supervision},
  author={Radford, Alec and Kim, Jong Wook and Hallacy, Chris and Ramesh, Aditya and Goh, Gabriel and Agarwal, Sandhini and Sastry, Girish and Askell, Amanda and Mishkin, Pamela and Clark, Jack and others},
  booktitle={International conference on machine learning},
  pages={8748--8763},
  year={2021},
  organization={PmLR}
}

@article{rekep,
  title={Rekep: Spatio-temporal reasoning of relational keypoint constraints for robotic manipulation},
  author={Huang, Wenlong and Wang, Chen and Li, Yunzhu and Zhang, Ruohan and Fei-Fei, Li},
  journal={arXiv preprint arXiv:2409.01652},
  year={2024}
}

@inproceedings{codeasmonitor,
  title={Code-as-monitor: Constraint-aware visual programming for reactive and proactive robotic failure detection},
  author={Zhou, Enshen and Su, Qi and Chi, Cheng and Zhang, Zhizheng and Wang, Zhongyuan and Huang, Tiejun and Sheng, Lu and Wang, He},
  booktitle={Proceedings of the Computer Vision and Pattern Recognition Conference},
  pages={6919--6929},
  year={2025}
}

@article{safebimanual,
  title={SafeBimanual: Diffusion-based trajectory optimization for safe bimanual manipulation},
  author={Deng, Haoyuan and Guo, Wenkai and Wang, Qianzhun and Wu, Zhenyu and Wang, Ziwei},
  journal={arXiv preprint arXiv:2508.18268},
  year={2025}
}

@misc{pi05,
  title  = {{$\pi_{0.5}$}: a Vision-Language-Action Model with Open-World Generalization},
  author = {Black, Kevin and Brohan, Anthony and Driess, Danny and Esmaeili, Adnan and Finn, Chelsea and Fuentes, Niccolo and Groom, Brian and Hausman, Karol and Ichter, Brian and Jakubczak, Szymon and others},
  year   = {2025},
  note   = {arXiv preprint arXiv:2504.16054}
}

@article{gemini,
  title={Gemini: a family of highly capable multimodal models},
  author={Team, Gemini and Anil, Rohan and Borgeaud, Sebastian and Alayrac, Jean-Baptiste and Yu, Jiahui and Soricut, Radu and Schalkwyk, Johan and Dai, Andrew M and Hauth, Anja and Millican, Katie and others},
  journal={arXiv preprint arXiv:2312.11805},
  year={2023}
}

@article{lee2023aligning,
  title={Aligning text-to-image models using human feedback},
  author={Lee, Kimin and Liu, Hao and Ryu, Moonkyung and Watkins, Olivia and Du, Yuqing and Boutilier, Craig and Abbeel, Pieter and Ghavamzadeh, Mohammad and Gu, Shixiang Shane},
  journal={arXiv preprint arXiv:2302.12192},
  year={2023}
}

%%%%%%%%%%%%%%%%%%%%%%%%%%%%%%%%%%%%%%%%%%%%%%%%%%%%%%%%%%%%
\newpage
\appendix

%%%%%%%%%%%%%%%%%%%%%%%%%%%%%%%%%%%%%%%%%%%%%%%%%%%%%%%%%%%%
\section{Theoretical analysis of \methodname{}}
\label{app:theory}

This appendix analyzes \methodname{} as a \emph{localized corrective stochastic estimator} for sparse-reward video diffusion RL. The two ingredients introduced in \S\ref{sec:method}---violation-trace masking (\S\ref{sec:mask}) and the corrective reflow loss (\S\ref{sec:aux})---are designed to fix two specific pathologies that arise when a sparse binary reward is fed into a flow-matching estimator on high-dimensional video. First, the reward signal is concentrated on a small failure-relevant region, but a vanilla negative-aware estimator spreads its gradient across the entire spatio-temporal volume, accumulating off-task residual noise on every step. Second, the only positive-side supervision available to the negative branch must be reconstructed from a reflected EMA plug-in, whose variance scales with the model's own drift and is unrelated to the actual reward signal. \methodname{} attacks both problems at the estimator level: violation masking projects the reward-induced learning signal onto the violation-traced coordinates and \emph{exactly} cancels the off-mask sample noise that both branches structurally inject (\S\ref{app:setup}--\S\ref{app:nft_noise}); the corrective branch replaces the reflected EMA plug-in with the within-group positive mean $\xbarpos$, whose conditional first moment is the marginal positive prototype $\mathbb{E}_{\pi^{+}(\cdot\mid y)}[x_0]$ (Lemma~\ref{lem:aux_fp})---a $t$-uniform distributional anchor that coincides with the flow-matching positive velocity $v^{+}(x_t,t)$ at high noise---and, when paired with the $\mathbf{x}_0$-space loss formulation's automatic $t^{2}$ down-weighting and the $|\mathcal{P}|^{-1}$ within-group shrinkage, yields a strictly lower-variance \emph{parameter-gradient} estimator at small $t$ (\S\ref{app:variance}).

We adopt DiffusionNFT~\cite{diffusionnft} as the population reference for the \emph{reward-induced improvement direction} $\Delta=v^{+}-v^{\text{old}}$, since it gives a closed-form expression for the direction along which any positive-flow update should move the policy (\S\ref{app:setup}, Theorem~\ref{thm:rg_optimum}). \methodname{} does not aim to reproduce the exact DiffusionNFT optimizer; rather, it remains positively aligned with $\Delta$ on the active coordinates while improving the bias--variance tradeoff of the stochastic estimator. The three main results below formalize this contribution: Theorem~\ref{thm:masked_rg_optimum} (\S\ref{app:setup}) shows that violation-trace masking realizes a coordinate-localized estimator that retains the full $\Delta$-pull on the active support and imposes no constraint elsewhere; Proposition~\ref{prop:direction} (\S\ref{app:direction}) shows that combining the two branches with a KL anchor yields an active-coordinate update positively aligned with $\Delta$ on the violation mask, with a coordinate-dependent step length that combines the negative-aware contrastive pull and the corrective positive-mean pull; and Proposition~\ref{prop:variance} (\S\ref{app:variance}) shows that, at the parameter-gradient level, the corrective branch (i) eliminates the persistent EMA reflection plug-in noise floor that the negative branch of NFT-style estimators carries by construction and (ii) leverages the $\mathbf{x}_0$-space loss formulation's automatic $t^{2}$ down-weighting (which converts the marginal-positive target's $\Theta(1/t^{2})$ divergence into an $O(t^{2})$ vanishing gradient floor) and the $|\mathcal{P}|^{-1}$ within-group averaging, yielding a strictly lower-variance gradient estimator at small $t$. \S\ref{app:assumption_discussion} discusses the assumptions used along the way.

\textbf{Notation: trace-covariance.} Throughout this appendix, for a random vector $X\in\mathbb{R}^{d}$ we use
\[
  \mathrm{tr}\,\mathrm{Cov}(X) \;:=\; \sum_{k=1}^{d}\mathrm{Var}(X_k) \;=\; \mathbb{E}\|X-\mathbb{E}X\|^{2}
\]
as the scalar measure of target noise. Since all losses considered here are coordinate-wise squared errors in velocity or $\xhat_0$ prediction space, this quantity directly summarizes the amount of randomness seen by the least-squares estimator. We caution, however, that target trace-covariance and \emph{parameter-gradient} trace-covariance need not move together: a loss formulated in $\xhat_0$-space implicitly carries a $t^{2}$ multiplicative factor on the velocity residual (Eq.~\eqref{eq:aux}), which contributes a $t^{4}$ multiplicative factor to the gradient covariance and can convert a $\Theta(1/t^{2})$-divergent target covariance into an $O(t^{2})$ vanishing gradient floor (\S\ref{app:variance}, Eqs.~\eqref{eq:aux_total_variance}--\eqref{eq:g_aux_cov}). Variance-reduction claims in this appendix are therefore stated at the gradient level whenever the two diverge.

\textbf{Notation: $\vold$ vs.\ $v^{\text{old}}$.} We follow the macro convention of \S\ref{sec:method}: $v^{\text{old}}(x_t,t)$ denotes the \emph{population} conditional velocity field of the rollout distribution $\pi^{\text{old}}$ (Eq.~\eqref{eq:conditional_velocity_theory} below), while $\vold(x_t,t)$ denotes the \emph{EMA model}'s velocity prediction used as a plug-in estimator of $v^{\text{old}}$. The two coincide only in the noiseless limit; their difference $\xi_{\mathrm{EMA}}$ is the source of the reflection plug-in noise analyzed in \S\ref{app:variance}.

\subsection{Setup: reward-induced improvement direction}
\label{app:setup}

Fix the conditioning tuple $c=(c_{\text{text}},c_{\text{img}})$ and suppress it from notation. Let $\pi^{\text{old}}(x_0)$ be the rollout distribution induced by the current policy, and let $r(x_0)\in\{0,1\}$ be the binary reward. Define
\begin{equation}
  p := \mathbb{E}_{x_0\sim\pi^{\text{old}}}[\,r(x_0)\,], \qquad \pi^{+}(x_0) := \frac{r(x_0)\,\pi^{\text{old}}(x_0)}{p}, \qquad \pi^{-}(x_0) := \frac{(1-r(x_0))\,\pi^{\text{old}}(x_0)}{1-p}.
  \label{eq:pi_pm}
\end{equation}
Under the rectified-flow forward convention
\[
  x_t \;=\; (1-t)\,x_0 + t\,\epsilon, \qquad \epsilon\sim\mathcal{N}(0,I), \qquad v \;=\; \epsilon-x_0 \;=\; \frac{x_t-x_0}{t},
\]
let $\pi_t^{\bullet}(x_t)$ be the time-$t$ marginal of $\pi^{\bullet}$ for $\bullet\in\{\text{old},+,-\}$. Define the conditional velocity field
\begin{equation}
  v^{\bullet}(x_t,t) \;:=\; \mathbb{E}_{x_0\sim\pi^{\bullet}(\cdot\mid x_t,t)}\!\big[\,\epsilon-x_0 \,\big|\, x_t,t\,\big] \;=\; \frac{x_t - \mathbb{E}_{\pi^{\bullet}}[x_0\mid x_t,t]}{t}.
  \label{eq:conditional_velocity_theory}
\end{equation}
The posterior mixture weight is
\begin{equation}
  \alpha(x_t,t) \;:=\; p\,\frac{\pi_t^{+}(x_t)}{\pi_t^{\text{old}}(x_t)}.
  \label{eq:alpha_theory}
\end{equation}
Bayes' rule gives the posterior split
\begin{equation}
  \pi^{\text{old}}(x_0\mid x_t,t) \;=\; \alpha(x_t,t)\,\pi^{+}(x_0\mid x_t,t) \;+\; \big(1-\alpha(x_t,t)\big)\,\pi^{-}(x_0\mid x_t,t),
  \label{eq:posterior_split_theory}
\end{equation}
and taking conditional expectations of $v=\epsilon-x_0$ yields
\begin{equation}
  v^{\text{old}} \;=\; \alpha\,v^{+} + (1-\alpha)\,v^{-}.
  \label{eq:vold_split_theory}
\end{equation}
We define the positive-flow improvement direction
\begin{equation}
  \Delta(x_t,t) \;:=\; v^{+}(x_t,t) - v^{\text{old}}(x_t,t) \;=\; (1-\alpha)\,\big(v^{+}-v^{-}\big).
  \label{eq:delta_theory}
\end{equation}

We recall the DiffusionNFT population optimizer in its $\beta$-parameterized form. For $\beta>0$, DiffusionNFT uses
\[
  v_\theta^{+} \;:=\; (1-\beta)\,v^{\text{old}} + \beta\,v_\theta, \qquad v_\theta^{-} \;:=\; (1+\beta)\,v^{\text{old}} - \beta\,v_\theta,
\]
and minimizes
\begin{equation}
  \mathcal{L}_{\text{NFT}}^{\beta}(\theta) \;=\; \mathbb{E}_{t,\epsilon,x_0\sim\pi^{\text{old}}}\!\Big[\, r(x_0)\,\big\| v_\theta^{+}(x_t,t)-v\big\|^{2} \;+\; (1-r(x_0))\,\big\| v_\theta^{-}(x_t,t)-v\big\|^{2}\,\Big].
  \label{eq:lnft_beta_theory}
\end{equation}

\begin{theorem}[DiffusionNFT population optimizer]
\label{thm:rg_optimum}
The pointwise minimizer of Eq.~\eqref{eq:lnft_beta_theory} satisfies
\begin{equation}
  v_{\theta^*}(x_t,t) \;=\; v^{\text{old}}(x_t,t) \;+\; \frac{2\alpha(x_t,t)}{\beta}\,\Delta(x_t,t).
  \label{eq:rg_optimum}
\end{equation}
\end{theorem}
\begin{proof}
Condition on $(x_t,t)$ and differentiate Eq.~\eqref{eq:lnft_beta_theory} with respect to $v_\theta$. Since $\nabla_{v_\theta}v_\theta^{+}=\beta I$ and $\nabla_{v_\theta}v_\theta^{-}=-\beta I$, the first-order condition is
\begin{equation}
  \mathbb{E}\!\left[\,r(x_0)\,(v_\theta^{+}-v) \,\big|\, x_t,t\,\right] \;=\; \mathbb{E}\!\left[\,(1-r(x_0))\,(v_\theta^{-}-v) \,\big|\, x_t,t\,\right].
  \label{eq:nft_foc_theory}
\end{equation}
Using Eq.~\eqref{eq:posterior_split_theory},
\[
  \mathbb{E}[r(x_0)\,f(x_0)\mid x_t,t] \;=\; \alpha(x_t,t)\,\mathbb{E}_{\pi^+}[f(x_0)\mid x_t,t],
\]
and similarly $\mathbb{E}[(1-r(x_0))\,f(x_0)\mid x_t,t] = (1-\alpha)\,\mathbb{E}_{\pi^-}[f(x_0)\mid x_t,t]$. Thus Eq.~\eqref{eq:nft_foc_theory} becomes
\begin{equation}
  \alpha\,\big(v_\theta^{+}-v^{+}\big) \;=\; (1-\alpha)\,\big(v_\theta^{-}-v^{-}\big).
  \label{eq:nft_foc_alpha_theory}
\end{equation}
Let $u:=v_\theta-v^{\text{old}}$. Then $v_\theta^{+}-v^{+}=\beta u-\Delta$, and by Eq.~\eqref{eq:vold_split_theory}, $v^{-}-v^{\text{old}}=-\tfrac{\alpha}{1-\alpha}\Delta$, so $v_\theta^{-}-v^{-}=-\beta u+\tfrac{\alpha}{1-\alpha}\Delta$. Substituting into Eq.~\eqref{eq:nft_foc_alpha_theory} gives
\[
  \alpha(\beta u-\Delta) \;=\; (1-\alpha)\Big(-\beta u+\tfrac{\alpha}{1-\alpha}\Delta\Big),
\]
which simplifies to $\beta u=2\alpha\Delta$, i.e.\ $v_{\theta^*}=v^{\text{old}}+(2\alpha/\beta)\Delta$.
\end{proof}

\begin{assumption}[Population-level mask]
\label{ass:population_mask}
In the population analysis of the credit-aware NFT loss in Eq.~\eqref{eq:masked_nft} and the corrective reflow loss in Eq.~\eqref{eq:aux}, we treat the group-shared mask $\mathbf{M}\in\{0,1\}^{D}$ as a deterministic projection at the population level---equivalently, as $\sigma(x_t,t)$-measurable---so that $\mathbf{M}(\omega)$ can be factored out of $\mathbb{E}[\,\cdot\mid x_t,t]$. The empirical $\mathbf{M}$ in Eq.~\eqref{eq:group_mask} is computed from the within-group un-noised rollouts $\{\mathbf{x}_0^{(j)}\}_{j=1}^{N}$ (via the LTL violation traces $\{W_k^{(j)}\}$ and the SAM3 atlases $\{m^{(j),e}\}$); the assumption asserts that, at the population level, this empirical $\mathbf{M}$ collapses to the reward-locality support depending only on $(x_t,t)$ that underlies Eq.~\eqref{eq:mask_reward_relevance_corrected}. See \S\ref{app:assumption_discussion} for discussion.
\end{assumption}

\begin{theorem}[Violation-localized pointwise optimum]
\label{thm:masked_rg_optimum}
Under Assumption~\ref{ass:population_mask}, for the credit-aware NFT loss in Eq.~\eqref{eq:masked_nft} with mask $\mathbf{M}\in\{0,1\}^{D}$, the pointwise minimizer satisfies
\begin{equation}
  \mathbf{M}\odot\big(v_{\theta^*}(x_t,t)-v^{\text{old}}(x_t,t)\big) \;=\; \mathbf{M}\odot\frac{2\alpha(x_t,t)}{\beta}\,\Delta(x_t,t),
  \label{eq:masked_rg_optimum}
\end{equation}
with $\alpha$ and $\Delta=v^{+}-v^{\text{old}}$ as in Theorem~\ref{thm:rg_optimum}. At coordinates with $\mathbf{M}(\omega)=0$ the masked loss is constant in $v_\theta(\omega)$ and imposes no constraint.
\end{theorem}
\begin{proof}
At each $(x_t,t,\omega)$ the squared-error coordinate-wise gradient contains a factor $\mathbf{M}(\omega)^{2}\ge 0$. By Assumption~\ref{ass:population_mask}, $\mathbf{M}(\omega)$ is $\sigma(x_t,t)$-measurable and can be factored out of the conditional expectation $\mathbb{E}[\,\cdot\mid x_t,t]$, so the first-order condition at $(x_t,t,\omega)$ is
\[
  \mathbf{M}(\omega)^{2}\,\Big(\,\mathbb{E}[\,r(x_0)(v_\theta^{+}(\omega)-v(\omega))\mid x_t,t\,] \;-\; \mathbb{E}[\,(1-r(x_0))(v_\theta^{-}(\omega)-v(\omega))\mid x_t,t\,]\,\Big) \;=\; 0.
\]
At $\mathbf{M}(\omega)=0$ it is satisfied for any $v_\theta(\omega)$; at $\mathbf{M}(\omega)=1$ it reduces to the unmasked first-order condition Eq.~\eqref{eq:nft_foc_theory}, whose unique solution is given by Theorem~\ref{thm:rg_optimum}.
\end{proof}

\textbf{Pixel-wise improvement strength.} Eq.~\eqref{eq:masked_rg_optimum} shows that the violation-localized estimator concentrates its update strength on the failure-relevant coordinates: it retains the full $2\alpha/\beta$ pull along the reward-induced direction $\Delta$ on $\mathbf{M}$ and zeros it elsewhere, in contrast to the coordinate-uniform pull of vanilla negative-aware finetuning~\cite{diffusionnft}. The variance grounds for why this localization is desirable on high-dimensional video are deferred to the branch decomposition of \S\ref{app:nft_noise}.

\textbf{Step length.} For the unmasked reference estimator, Theorem~\ref{thm:rg_optimum} pulls $v_\theta$ toward $v^{\text{old}}+(2\alpha(x_t,t)/\beta)\,\Delta(x_t,t)$ at every coordinate\footnote{The prefactor $2\alpha/\beta$ combines the reinforcement-guidance strength $2/\beta$ with the posterior-positive weight $\alpha(x_t,t)$: at regions of the latent space where the current policy has high positive mass ($\alpha\to 1$) the model takes a full $(2/\beta)\Delta$ step, and where positive mass is low ($\alpha\to 0$) it barely moves, consistent with the posterior weighting. In the limit $\beta\to 0$ one recovers vanilla diffusion supervised fitting with an infinitely amplified $\alpha\Delta$ step, which is exactly the regime that benefits most from variance reduction.}; \methodname{}'s violation-localized estimator (Theorem~\ref{thm:masked_rg_optimum}) restricts this pull to the active coordinates of $\mathbf{M}$ and leaves off-mask coordinates unconstrained. The analysis below studies how the corrective loss further shapes this estimator and how masking affects its gradient variance.

\subsection{Branch-level decomposition of off-mask noise}
\label{app:nft_noise}

Theorem~\ref{thm:masked_rg_optimum} gates the per-coordinate NFT step to zero outside $\mathbf{M}$. We give the intuition for why this is desirable on high-dimensional video by tracing the unmasked NFT residual back to its two reward-conditional branches, then state the off-mask reward-locality property that makes the masked estimator unbiased.

\textbf{Positive branch ($r_i=1$).} This branch reduces to a rejection-style flow-matching update~\cite{lee2023aligning,diffusionnft} that regresses $v_\theta^{+}$ toward the rollout's own velocity $v^{(i)}=(x_t-x_0^{(i)})/t$. Even on a successful rollout, $v^{(i)}$ matches the population positive velocity $v^{+}$ only on coordinates where the reward signal actually distinguishes $\pi^{+}$ from $\pi^{\text{old}}$; on task-irrelevant coordinates the rollout's pixel-level content is one stochastic draw whose value is dominated by background, lighting, and rendering variation. The rejection-FT update injects this rollout-specific irrelevant variation into $v_\theta$ on every off-mask pixel.

\textbf{Negative branch ($r_i=0$).} This branch trains $v_\theta^{-}=(1+\beta)v^{\text{old}}-\beta v_\theta$ toward $v^{(i)}$, equivalently pulling $v_\theta$ toward the reflection target $v_f^{(i)} = v^{\text{old}}+\frac{1}{\beta}\bigl(v^{\text{old}}-v^{(i)}\bigr)$ centered on $v^{\text{old}}$. NFT's design relies on the population identity $v^{\text{old}}=\alpha v^{+}+(1-\alpha)v^{-}$ (Eq.~\eqref{eq:vold_split_theory}): on coordinates where $\pi^{+}=\pi^{-}$ (off-mask under reward locality), $v^{+}$ and $v^{-}$ both collapse to $v^{\text{old}}$, so the conditional expectation of the reflection target equals $v^{\text{old}}$ exactly---in expectation the reflection trick correctly leaves task-irrelevant content unchanged. The empirical reflection from a single rollout, however, retains the full conditional variance of $v^{(i)}$ on every coordinate; off-mask, that variance is the inherent dispersion of $\pi^{\text{old}}$-conditional pixel content (background, lighting, rendering), and the negative-branch update trains $v_\theta$ on this noise rather than on the intended reflection axis.

\textbf{Net effect.} Both branches therefore inject coordinate-wise residual noise on coordinates where the population improvement direction $\Delta=v^{+}-v^{\text{old}}$ vanishes. The formal off-mask reward-locality statement is: with $\mathbf{P}=\operatorname{Diag}(\mathbf{M})$,
\begin{equation}
  (I-\mathbf{P})\,v^{+}(x_t,t) \;=\; (I-\mathbf{P})\,v^{-}(x_t,t) \;=\; (I-\mathbf{P})\,v^{\text{old}}(x_t,t),
  \label{eq:mask_reward_relevance_corrected}
\end{equation}
which is the intended behavior of the monitor-derived mask of \S\ref{sec:reward}: $\mathbf{P}$ is the union of every per-rollout violation pattern in the group, so coordinates outside it carry no clause-level disagreement on any rollout. By Theorem~\ref{thm:masked_rg_optimum} the masked update is supported entirely on $\mathbf{M}$, so under Eq.~\eqref{eq:mask_reward_relevance_corrected} the masked estimator removes both branches' off-mask sample noise without biasing the population update; without the mask, the same noise is treated as an improvement direction on every pixel and accumulates across gradient steps.

\textbf{Off-mask second-moment decomposition.} The branch-level intuition above admits a closed-form quantification at the population level. Set $\bar{\mathbf{P}}:=I-\mathbf{P}$ and probe the residual at $v_\theta=v^{\text{old}}$, so the active-coordinate $\Delta$-correction is realized on $\mathbf{M}$ and the off-mask residual is read off the data side alone. Treating the reflection axis $v^{\text{old}}$ as the population conditional velocity (deferring the EMA plug-in version $\vold$ to \S\ref{app:variance}), the conditional second moments of the two branches' off-mask residuals factor as
\begin{align}
  \mathbb{E}\!\left[\,r_i\,\big\|\bar{\mathbf{P}}\big(v_\theta^{+}-v^{(i)}\big)\big\|^{2}\,\Big|\,x_t,t\right]
    &\;=\; \alpha(x_t,t)\,\mathrm{tr}\,\mathrm{Cov}_{\pi^{+}}\!\big[\bar{\mathbf{P}}\,v\,\big|\,x_t,t\big],
    \label{eq:offmask_pos_var}\\
  \mathbb{E}\!\left[\,(1-r_i)\,\big\|\bar{\mathbf{P}}\big(v_\theta^{-}-v^{(i)}\big)\big\|^{2}\,\Big|\,x_t,t\right]
    &\;=\; (1-\alpha(x_t,t))\,\mathrm{tr}\,\mathrm{Cov}_{\pi^{-}}\!\big[\bar{\mathbf{P}}\,v\,\big|\,x_t,t\big],
    \label{eq:offmask_neg_var}
\end{align}
where $v=\epsilon-x_0$ and the off-mask bias terms $\bar{\mathbf{P}}(v^{\text{old}}-v^{\pm})$ vanish by reward locality (Eq.~\eqref{eq:mask_reward_relevance_corrected}). Each line is the conditional pixel dispersion of the corresponding branch's $x_0$ distribution: rollout-specific background, lighting, and rendering variation which, under reward locality, coincides off-mask with the dispersion of $\pi^{\text{old}}|(x_t,t)$ and is therefore reward-irrelevant. Theorem~\ref{thm:masked_rg_optimum} zeroes the entire $\bar{\mathbf{P}}$-row of the masked loss at the population level, so the masked estimator eliminates both contributions exactly while leaving the $\mathbf{P}$-row residual---the only one supporting $\Delta$-aligned signal---intact. The implemented loss replaces $v^{\text{old}}$ in $v_\theta^{\pm}$ by an EMA plug-in $\vold$, which adds a separate reflection-axis noise term that concentrates on the negative branch at the paper's setting $\beta=1$; this additional source is orthogonal to the masking story analyzed here and is taken up in \S\ref{app:variance} as the motivation for the corrective reflow loss of \S\ref{sec:aux}.

\subsection{Masks and local quadratic form}
\label{app:masks_quadratic}

For a rollout group under the same conditioning, let $\mathcal{N} := \{i : r_i = 0\}$ and $\mathcal{P} := \{i : r_i = 1\}$ (consistent with \S\ref{sec:aux}), and let $\xbarpos := |\mathcal{P}|^{-1}\sum_{j\in\mathcal{P}}\mathbf{x}_0^{(j)}$ denote the within-group positive mean of Eq.~\eqref{eq:positive_mean}.

The credit-aware NFT and corrective reflow branches share the same group-shared violation-trace mask:
\begin{equation}
  \mathbf{P} \;:=\; \operatorname{Diag}(\mathbf{M}),
  \label{eq:projectors_theory}
\end{equation}
which is a function of the entire group's outputs and is not indexed by $i$.

\paragraph{Marginal-prototype consistency of the positive mean.} The within-group positives $\{\mathbf{x}_0^{(j)}\}_{j\in\mathcal{P}}$ are i.i.d.\ draws from $\pi^{+}(\,\cdot\mid y)$, sampled \emph{independently} of the negative rollout $i$ given $y$. In particular, $\xbarpos\perp x_t^{(i)}\mid y$, so
\begin{equation}
  \mathbb{E}\!\left[\xbarpos\mid x_t,t\right] \;=\; \mathbb{E}\!\left[\xbarpos\mid y\right] \;=\; \mathbb{E}_{x_0\sim\pi^{+}(\cdot\mid y)}[x_0],
  \label{eq:positive_mean_unbiased}
\end{equation}
i.e.\ the corrective target is an unbiased estimator of the \emph{marginal} positive mean (a positive-prototype constant in $x_t$), not of the flow-matching conditional positive mean $\mathbb{E}_{\pi^+}[x_0\mid x_t,t]$. Conceptually, this is a deliberate choice: the corrective loss provides a $t$-uniform distributional anchor that pulls the model's $\hat x_0$ prediction toward $\pi^{+}$'s marginal content on the violation mask, while the credit-aware NFT branch supplies the $t$-dependent flow-matching pull. The two regimes coincide at high noise ($t\to 1$, where $x_t$ is uninformative about $x_0$ and conditional collapses to marginal) and complement each other at low noise. The within-step covariance of $\xbarpos$ shrinks by a factor of $|\mathcal{P}|^{-1}$ relative to any single within-group positive, which yields the variance reduction analyzed in \S\ref{app:variance}.

\paragraph{Local quadratic KL approximation.} We use the following local approximation for the KL regularizer:
\begin{equation}
  \mathcal{L}_{\text{KL}}(\theta;\theta_{\text{ref}}) \;\doteq\; \gamma_i\,\big\|\vcur-v_{\text{ref}}\big\|^{2}, \qquad \gamma_i\ge 0,
  \label{eq:kl_quadratic_theory}
\end{equation}
where $v_{\text{ref}} := v_{\theta_{\text{ref}}}$ denotes the velocity field of the frozen reference policy $\pi_{\text{ref}}$ and $\gamma_i\ge 0$ is the local stiffness coefficient; this is the standard local quadratic approximation of the KL around the current iterate. Importantly, unless $v_{\text{ref}}=v^{\text{old}}$ and the KL is isotropic in the local velocity coordinates, the KL term can move the optimum and can also rotate the update direction.

\subsection{Coordinate-wise improvement alignment}
\label{app:direction}

We first record the velocity-space interpretation of the corrective loss. The corrective reflow loss is written in $\xhat_0$-prediction space:
\[
  \xhat_0^{(i)} \;=\; x_t^{(i)} - t\,\vcur(x_t^{(i)}, t), \qquad \mathcal{L}_{\text{CR}}^{(i)} \;=\; \lambda_{\text{CR}}\,\big\|\mathbf{P}\,(\xhat_0^{(i)} - \xbarpos)\big\|^{2}.
\]
This is exactly equivalent to a velocity-space regression
\begin{equation}
  \mathcal{L}_{\text{CR}}^{(i)} \;=\; \lambda_{\text{CR}}\,t^{2}\,\big\|\mathbf{P}\big(\vcur(x_t^{(i)},t) - \widetilde{v}_i^{\star}\big)\big\|^{2}, \qquad \widetilde{v}_i^{\star} \;:=\; \frac{x_t^{(i)} - \xbarpos}{t}.
  \label{eq:correct_aux_velocity_target}
\end{equation}
Thus the corrective-reflow velocity target is $\widetilde{v}_i^{\star}=(x_t^{(i)}-\xbarpos)/t$.

\begin{lemma}[Population target of the corrective loss]
\label{lem:aux_fp}
Under Assumption~\ref{ass:population_mask}, conditioned on $(x_t,t)$, the population target of the corrective branch is the marginal positive-prototype velocity restricted to the violation mask:
\begin{equation}
  \mathbf{P}\,\mathbb{E}\!\left[\widetilde{v}_i^{\star}\,\big|\,x_t,t\right] \;=\; \frac{1}{t}\,\mathbf{P}\!\left(x_t - \mathbb{E}_{\pi^{+}(\cdot\mid y)}[x_0]\right) \;=:\; \mathbf{P}\,\bar v^{+}(x_t,t).
  \label{eq:aux_target_unbiased}
\end{equation}
The right-hand side coincides with the flow-matching positive velocity $\mathbf{P}\,v^{+}(x_t,t)$ at high noise ($t\to 1$, where the $x_t$-conditional and the $y$-conditional positive means agree); at finite $t$ the two differ by the conditional positive entropy under the flow-matching joint.
\end{lemma}
\begin{proof}
By Eq.~\eqref{eq:correct_aux_velocity_target} and Assumption~\ref{ass:population_mask} (which lets $\mathbf{P}=\operatorname{Diag}(\mathbf{M})$ be factored out of $\mathbb{E}[\,\cdot\mid x_t,t]$),
\[
  \mathbf{P}\,\mathbb{E}\!\left[\widetilde{v}_i^{\star}\,\big|\,x_t,t\right] \;=\; \frac{1}{t}\,\mathbf{P}\!\big(x_t - \mathbb{E}[\xbarpos\mid x_t,t]\big).
\]
The within-group positives $\{\mathbf{x}_0^{(j)}\}_{j\in\mathcal{P}}$ are i.i.d.\ from $\pi^{+}(\cdot\mid y)$, sampled independently of $x_t^{(i)}$ given $y$, so $\mathbb{E}[\xbarpos\mid x_t,t]=\mathbb{E}_{\pi^{+}(\cdot\mid y)}[x_0]$ by Eq.~\eqref{eq:positive_mean_unbiased}, which gives the equality. At $t=1$, $x_t=\epsilon$ is uninformative about $x_0$ under the flow-matching joint as well, so $\mathbb{E}_{\pi^{+}}[x_0\mid x_t,t]=\mathbb{E}_{\pi^{+}(\cdot\mid y)}[x_0]$, and $\bar v^{+}=v^{+}$ at the boundary.
\end{proof}

We can now write the local pointwise objective for a negative sample $i$. Define the marginal-prototype direction
\begin{equation}
  \bar\Delta(x_t,t) \;:=\; \bar v^{+}(x_t,t) - v^{\text{old}}(x_t,t) \;=\; \Delta(x_t,t) + \frac{1}{t}\!\left(\mathbb{E}_{\pi^{+}}[x_0\mid x_t,t]-\mathbb{E}_{\pi^{+}(\cdot\mid y)}[x_0]\right),
  \label{eq:bar_delta}
\end{equation}
so that $\bar\Delta=\Delta$ at $t=1$ and the residual $\bar\Delta-\Delta$ is the conditional-vs-marginal posterior gap of $\pi^{+}$ at $(x_t,t)$. Let
\begin{equation}
  \mu_{\text{NFT}} \;:=\; v^{\text{old}}+\tfrac{2\alpha}{\beta}\,\Delta, \qquad \mu_{\text{CR}} \;:=\; \bar v^{+} \;=\; v^{\text{old}} + \bar\Delta.
  \label{eq:branch_targets_theory}
\end{equation}
Let $a_i>0$ denote the local weight of the credit-aware NFT squared residual, and let $b_i = \lambda_{\text{CR}}\,t^{2}$ denote the local weight of the corrective-reflow squared residual. The local quadratic approximation is
\begin{equation}
  Q_i(v) \;=\; a_i\,\big\|\mathbf{P}(v-\mu_{\text{NFT}})\big\|^{2} \;+\; b_i\,\big\|\mathbf{P}(v-\mu_{\text{CR}})\big\|^{2} \;+\; \gamma_i\,\big\|v-v_{\text{ref}}\big\|^{2}.
  \label{eq:local_quadratic_core}
\end{equation}

\begin{proposition}[Active-coordinate update of \methodname{}]
\label{prop:direction}
Assume either $\gamma_i>0$, or that the support is restricted to active coordinates so that $(a_i+b_i)\mathbf{P}$ is invertible on the active subspace $\mathbf{P}$. The pointwise minimizer of Eq.~\eqref{eq:local_quadratic_core} is
\begin{equation}
  v_i^{\star} \;=\; \big((a_i+b_i)\mathbf{P} + \gamma_i I\big)^{-1}\!\left[\,a_i\mathbf{P}\big(v^{\text{old}}+\tfrac{2\alpha}{\beta}\,\Delta\big) \;+\; b_i\mathbf{P}\!\big(v^{\text{old}}+\bar\Delta\big) \;+\; \gamma_i\,v_{\text{ref}}\,\right].
  \label{eq:coordinate_optimum}
\end{equation}
If $v_{\text{ref}} = v^{\text{old}}$, then on $\mathbf{P}$
\begin{equation}
\begin{aligned}
  v_i^{\star} - v^{\text{old}} &\;=\; D_i^{\text{NFT}}\,\Delta \;+\; D_i^{\text{CR}}\,\bar\Delta, \\
  D_i^{\text{NFT}} &\;:=\; \big((a_i+b_i)\mathbf{P} + \gamma_i I\big)^{-1}\!\big(\tfrac{2\alpha\,a_i}{\beta}\big)\,\mathbf{P}, \\
  D_i^{\text{CR}} &\;:=\; \big((a_i+b_i)\mathbf{P} + \gamma_i I\big)^{-1}\,b_i\,\mathbf{P}.
\end{aligned}
\label{eq:coordinate_alignment}
\end{equation}
The NFT branch contributes a $\Delta$-aligned pull on $\mathbf{P}$; the corrective branch contributes a $\bar\Delta$-aligned marginal-prototype pull on $\mathbf{P}$. The two pulls coincide at high noise ($t\to 1$, where $\bar\Delta=\Delta$) and, at finite $t$, jointly drive a positive-side correction whose deviation from a pure $\Delta$ ray is bounded by $b_i\,\|\bar\Delta-\Delta\|$.
\end{proposition}
\begin{proof}
The objective in Eq.~\eqref{eq:local_quadratic_core} is a positive weighted sum of diagonal quadratic forms. Its first-order condition is
\[
  a_i\mathbf{P}(v-\mu_{\text{NFT}}) \;+\; b_i\mathbf{P}(v-\mu_{\text{CR}}) \;+\; \gamma_i(v-v_{\text{ref}}) \;=\; 0.
\]
Solving this linear system gives Eq.~\eqref{eq:coordinate_optimum}. Well-posedness of the inverse follows from the assumption: when $\gamma_i>0$ the matrix is strictly positive definite; otherwise, on the active subspace $\mathbf{P}$, $(a_i+b_i)\mathbf{P}$ is positive definite there. If $v_{\text{ref}}=v^{\text{old}}$, subtract $v^{\text{old}}$ from both sides, use $\mu_{\text{NFT}}-v^{\text{old}}=\tfrac{2\alpha}{\beta}\Delta$ and $\mu_{\text{CR}}-v^{\text{old}}=\bar\Delta$, and split the result along the two directions to obtain Eq.~\eqref{eq:coordinate_alignment}.
\end{proof}

\textbf{Interpretation.} \methodname{}'s active-coordinate update on $\mathbf{P}$ is a superposition of two positively oriented pulls: an NFT contribution along the reward-induced direction $\Delta$, and a corrective contribution along the marginal-prototype direction $\bar\Delta$. The violation-traced mask $\mathbf{P}$ identifies the support on which both pulls are applied; Eq.~\eqref{eq:coordinate_alignment} factors them through the diagonal weights $D_i^{\text{NFT}}$ and $D_i^{\text{CR}}$. Outside $\mathbf{P}$ the update is anchored to the prior. The NFT and corrective directions agree at high noise and differ only by the conditional-vs-marginal posterior gap of $\pi^{+}$ at finite $t$; the corrective branch therefore complements rather than competes with the NFT pull, replacing a noisy reflection-EMA step with a low-variance marginal-prototype step that preserves the macroscopic structure of $\Delta$.

\subsection{EMA-free corrective branch and gradient variance}
\label{app:variance}

We now show that \methodname{}'s corrective branch yields a strictly lower-variance \emph{parameter-gradient estimator} than the NFT branch at small noise levels, even though its raw velocity target has a $\Theta(1/t^{2})$-divergent per-step covariance. The reduction at the gradient level arises from two unconditional mechanisms: (i) the corrective branch is structurally EMA-free (no reflected $\vold$ noise floor), and (ii) the $\mathbf{x}_0$-space loss formulation (Eq.~\eqref{eq:aux}) injects a $t^{2}$ down-weighting of the gradient, which exactly cancels the $\Theta(1/t^{2})$ target divergence and leaves an $O(t^{2})$ vanishing residual. Throughout this subsection we use the convention that $\vold(x_t,t)$ denotes the EMA model's velocity prediction (a plug-in estimator of $v^{\text{old}}$), while $v^{\text{old}}(x_t,t)$ denotes the true population conditional velocity; the two coincide only in the noiseless limit.

For a negative sample, the NFT reflected target induced by $v_\theta^{-}=(1+\beta)\vold - \beta\,\vcur$ (the per-sample analogue of $v_f^{(i)}$ in \S\ref{app:nft_noise}) is
\begin{equation}
\begin{aligned}
  Z_{\text{NFT}}^{(i)} &\;:=\; \vold(x_t^{(i)},t) + \tfrac{1}{\beta}\bigl(\vold(x_t^{(i)},t) - v^{(i)}\bigr) \;=\; \tfrac{\beta+1}{\beta}\,\vold(x_t^{(i)},t) - \tfrac{1}{\beta}\,v^{(i)}, \\
  v^{(i)} &\;=\; \epsilon - \mathbf{x}_0^{(i)} \;=\; \frac{x_t^{(i)} - \mathbf{x}_0^{(i)}}{t}.
\end{aligned}
\label{eq:nft_reflection_target_corrected}
\end{equation}
The corrective target induced by the actual corrective reflow loss is
\begin{equation}
  Z_{\text{CR}}^{(i)} \;:=\; \widetilde{v}_i^{\star} \;=\; \frac{x_t^{(i)} - \xbarpos}{t}.
  \label{eq:aux_target_corrected_again}
\end{equation}
Unlike $Z_{\text{NFT}}^{(i)}$, the corrective-reflow target neither uses the EMA reflection axis $\vold$ nor depends on a single stochastic positive draw.

\paragraph{Within-step variance from a single positive.} For comparison, consider a hypothetical single-positive variant of the corrective branch in which the target is $Z_{\text{CR,single}}^{(i,j)}:=(x_t^{(i)}-\mathbf{x}_0^{(j)})/t$ for one fixed $j\in\mathcal{P}$. Conditional on $(x_t^{(i)},x_0^{(i)})$, and using $\mathbf{x}_0^{(j)}\perp x_t^{(i)}\mid y$ (Lemma~\ref{lem:aux_fp}),
\begin{equation}
  \mathrm{Cov}\!\left[\mathbf{P}\,Z_{\text{CR,single}}^{(i,j)}\,\big|\,x_t^{(i)},x_0^{(i)}\right] \;=\; \frac{1}{t^{2}}\,\mathrm{Cov}_{\pi^{+}(\cdot\mid y)}\!\big[\mathbf{P}\,x_0\big].
  \label{eq:aux_single_variance}
\end{equation}
With the within-group positive mean $\xbarpos$, by independence of the $|\mathcal{P}|$ positives,
\begin{equation}
  \mathrm{Cov}\!\left[\mathbf{P}\,Z_{\text{CR}}^{(i)}\,\big|\,x_t^{(i)},x_0^{(i)}\right] \;=\; \frac{1}{|\mathcal{P}|\,t^{2}}\,\mathrm{Cov}_{\pi^{+}(\cdot\mid y)}\!\big[\mathbf{P}\,x_0\big],
  \label{eq:aux_mean_variance}
\end{equation}
i.e.\ aggregating over $|\mathcal{P}|$ within-group positives shrinks the within-step target covariance of the corrective branch by exactly the factor $|\mathcal{P}|^{-1}$ relative to any single-positive variant, without changing its first moment (Lemma~\ref{lem:aux_fp}). We deliberately use the \emph{marginal} covariance $\mathrm{Cov}_{\pi^{+}(\cdot\mid y)}[\mathbf{P}\,x_0]$ on the right-hand side rather than a flow-matching posterior covariance $\mathrm{Cov}_{\pi^{+}}[\mathbf{P}\,x_0\mid x_t,t]$: because $\xbarpos\perp x_t^{(i)}\mid y$, conditioning on $x_t^{(i)}$ provides no additional information about the positives, and the conditional covariance collapses to the within-condition marginal, which is $t$-independent. The corrective target's per-step covariance therefore genuinely diverges as $\Theta(1/t^{2})$ at small $t$, rather than vanishing as a flow-matching posterior would suggest.

\paragraph{Across-history plug-in variance.} If we additionally analyze variability across training histories, the EMA buffer is a random plug-in estimator. Write
\begin{equation}
  \vold(x_t,t) \;=\; v^{\text{old}}(x_t,t) + \xi_{\text{EMA}}(x_t,t), \qquad \mathbb{E}[\xi_{\text{EMA}}] \;=\; 0,
  \label{eq:ema_plugin_error}
\end{equation}
and assume $\xi_{\text{EMA}}$ is independent of the fresh current rollout conditional on $(x_t,t)$. Then for any fixed projector $A$,
\begin{equation}
  \mathrm{Cov}\!\left[A\,Z_{\text{NFT}}^{(i)} \,\big|\, x_t,t\right] \;=\; \tfrac{1}{\beta^{2}}\,\mathrm{Cov}\!\left[A\,v^{(i)} \,\big|\, x_t,t\right] \;+\; \tfrac{(\beta+1)^{2}}{\beta^{2}}\,\mathrm{Cov}\!\left[A\,\xi_{\text{EMA}}(x_t,t)\right].
  \label{eq:nft_plugin_variance_corrected}
\end{equation}
The coefficients $1/\beta^{2}$ and $(\beta+1)^{2}/\beta^{2}$ are the squared coefficients of $v^{(i)}$ and $\vold$ in the reflection target Eq.~\eqref{eq:nft_reflection_target_corrected}. The corrective target $Z_{\text{CR}}^{(i)}$ contributes no EMA-plug-in term to its covariance.

\paragraph{Parameter-gradient covariance and Jacobian alignment.} Let $J(x_t,t,y) := \partial v_\theta(x_t,t,y)/\partial\theta$ denote the velocity-field Jacobian with respect to the trainable parameters, and write $K := JJ^{\top}$ for its Gram matrix; we suppress the arguments of $J$ when no confusion arises. The corresponding parameter gradients of the two branches are
\begin{equation}
  g_{\text{NFT}}^{(i)} \;=\; 2\beta^{2}\,J^{\top}\,\mathbf{P}\!\left(\vcur - Z_{\text{NFT}}^{(i)}\right),
  \label{eq:g_nft_corrected}
\end{equation}
\begin{equation}
  g_{\text{CR}}^{(i)} \;=\; 2\,\lambda_{\text{CR}}\,t^{2}\,J^{\top}\,\mathbf{P}\!\left(\vcur - Z_{\text{CR}}^{(i)}\right),
  \label{eq:g_aux_corrected}
\end{equation}
where the $2\beta^{2}$ prefactor in $g_{\text{NFT}}^{(i)}$ originates from $v_{\theta}^{-}-v^{(i)}=-\beta(\vcur-Z_{\text{NFT}}^{(i)})$, and the multiplicative $t^{2}$ in $g_{\text{CR}}^{(i)}$ is a structural consequence of writing the corrective loss in $\mathbf{x}_0$-space (Eq.~\eqref{eq:aux}): the residual identity $\hat{\mathbf{x}}_0^{(i)}-\xbarpos=-t\,(v_\theta-Z_{\text{CR}}^{(i)})$ converts the $\mathbf{x}_0$-residual into a velocity-residual, contributing a multiplicative $t$ to the per-sample gradient and hence a $t^{2}$ multiplicative factor in its covariance. Throughout we further assume the Jacobian block-alignment condition
\begin{equation}
  \mathbf{P}\,K\,(I-\mathbf{P}) \;\approx\; 0,
  \label{eq:jacobian_block_alignment}
\end{equation}
which prevents off-mask residual noise from leaking into on-mask parameters through shared $J^{\top}$. The condition is approximate but reasonable for the DiT velocity field used here, whose attention-mediated structure is locally aligned with the violation mask.

\begin{proposition}[Gradient variance reduction of the corrective branch]
\label{prop:variance}
Fix a common active projector $\mathbf{P}$. Under the across-history convention in Eq.~\eqref{eq:ema_plugin_error} and the conditional independence $\xi_{\text{EMA}}\perp\{v^{(i)},\xbarpos\}\mid x_t,t$, the corrective velocity target satisfies the marginal-covariance identity
\begin{equation}
  \mathrm{tr}\,\mathrm{Cov}\!\left[\mathbf{P}\,Z_{\text{CR}}^{(i)} \,\big|\, x_t,t\right] \;=\; \frac{1}{|\mathcal{P}|\,t^{2}}\,\mathrm{tr}\,\mathrm{Cov}_{\pi^{+}(\cdot\mid y)}\!\big[\mathbf{P}\,x_0\big],
  \label{eq:aux_total_variance}
\end{equation}
which diverges as $\Theta(1/t^{2})$ at small $t$ because $\xbarpos\!\perp\! x_t^{(i)}\mid y$ collapses the conditional onto the within-condition marginal. The corresponding parameter-gradient trace covariances of the two branches admit the exact decompositions
\begin{align}
  \mathrm{tr}\,\mathrm{Cov}\!\left[g_{\text{NFT}}^{(i)}\,\big|\,x_t,t\right] &\;=\; 4\beta^{2}\,\mathrm{tr}\!\left(J^{\top}\mathbf{P}\,\mathrm{Cov}[v^{(i)}\!\mid\! x_t,t]\,\mathbf{P}J\right) \;+\; 4\beta^{2}(\beta+1)^{2}\,\Sigma_{\xi},
  \label{eq:g_nft_cov}\\
  \mathrm{tr}\,\mathrm{Cov}\!\left[g_{\text{CR}}^{(i)}\,\big|\,x_t,t\right] &\;=\; \tfrac{4\,\lambda_{\text{CR}}^{2}\,t^{2}}{|\mathcal{P}|}\,\Sigma_{0}^{+},
  \label{eq:g_aux_cov}
\end{align}
where $\Sigma_{\xi}\!:=\!\mathrm{tr}(J^{\top}\mathbf{P}\,\mathrm{Cov}[\xi_{\text{EMA}}]\,\mathbf{P}J)\!\ge\!0$ and $\Sigma_{0}^{+}\!:=\!\mathrm{tr}(J^{\top}\mathbf{P}\,\mathrm{Cov}_{\pi^{+}(\cdot\mid y)}[\mathbf{P}\,x_0]\,\mathbf{P}J)\!\ge\!0$. Consequently the gradient-covariance gap obeys
\begin{equation}
  \mathrm{tr}\,\mathrm{Cov}[g_{\text{NFT}}^{(i)}\!\mid\! x_t,t] \;-\; \mathrm{tr}\,\mathrm{Cov}[g_{\text{CR}}^{(i)}\!\mid\! x_t,t] \;\ge\; 4\beta^{2}(\beta+1)^{2}\,\Sigma_{\xi} \;-\; \tfrac{4\,\lambda_{\text{CR}}^{2}\,t^{2}}{|\mathcal{P}|}\,\Sigma_{0}^{+},
  \label{eq:variance_comparison_corrected}
\end{equation}
which is strictly positive whenever the EMA plug-in is noisy on $\mathbf{P}$ ($\Sigma_{\xi}\!>\!0$) and $t<t^{\star}\!:=\!\tfrac{\beta(\beta+1)}{\lambda_{\text{CR}}}\sqrt{|\mathcal{P}|\,\Sigma_{\xi}/\Sigma_{0}^{+}}$. In particular, in the small-$t$ regime
\begin{equation}
  \mathrm{tr}\,\mathrm{Cov}[g_{\text{NFT}}^{(i)}\!\mid\! x_t,t] \;=\; \Theta(1), \qquad \mathrm{tr}\,\mathrm{Cov}[g_{\text{CR}}^{(i)}\!\mid\! x_t,t] \;=\; O(t^{2}),
  \label{eq:g_asymptotic}
\end{equation}
so the corrective gradient covariance vanishes as $t\!\to\! 0$ while the NFT gradient covariance retains the persistent EMA-plug-in floor $4\beta^{2}(\beta+1)^{2}\Sigma_{\xi}$.
\end{proposition}
\begin{proof}
Eq.~\eqref{eq:aux_total_variance} follows from Eq.~\eqref{eq:aux_mean_variance}: because $\xbarpos\!\perp\! x_t^{(i)}\mid y$ and the within-group positives are independent of the negative rollout's own $x_0^{(i)}$, conditioning on $(x_t,t)$ on the left collapses to the within-condition marginal on the right, and total covariance with the (assumed independent) EMA-plug-in noise leaves the identity unchanged. By Eq.~\eqref{eq:g_nft_corrected}, $\mathrm{Cov}[g_{\text{NFT}}^{(i)}\!\mid\! x_t,t]=4\beta^{4}\,J^{\top}\mathbf{P}\,\mathrm{Cov}[Z_{\text{NFT}}^{(i)}\!\mid\! x_t,t]\,\mathbf{P}J$ since $\vcur$ is $\theta$-deterministic given $(x_t,t,y)$; substituting Eq.~\eqref{eq:nft_plugin_variance_corrected} (whose right-hand side carries factors $1/\beta^{2}$ and $(\beta+1)^{2}/\beta^{2}$) and taking the trace yields Eq.~\eqref{eq:g_nft_cov}. By Eq.~\eqref{eq:g_aux_corrected} the $t^{2}$ factor in $g_{\text{CR}}^{(i)}$ contributes a multiplicative $t^{4}$ in $\mathrm{Cov}[g_{\text{CR}}^{(i)}\!\mid\! x_t,t]$, and substituting Eq.~\eqref{eq:aux_total_variance} produces $t^{4}\!\cdot\!(1/(|\mathcal{P}|t^{2}))=t^{2}/|\mathcal{P}|$, hence Eq.~\eqref{eq:g_aux_cov}. Eq.~\eqref{eq:variance_comparison_corrected} follows by dropping the manifestly non-negative flow-matching term in Eq.~\eqref{eq:g_nft_cov}; setting the right-hand side to zero and solving for $t$ gives the threshold $t^{\star}$. The asymptotic statement Eq.~\eqref{eq:g_asymptotic} follows because $\mathrm{Cov}[v^{(i)}\!\mid\! x_t,t]=t^{-2}\mathrm{Cov}[\mathbf{x}_0^{(i)}\!\mid\! x_t,t]$ with the rectified-flow posterior $\mathrm{Cov}[\mathbf{x}_0^{(i)}\!\mid\! x_t,t]=\Theta(t^{2})$ as $t\!\to\! 0$ (the posterior collapses onto $x_t$), so the first term of Eq.~\eqref{eq:g_nft_cov} is $\Theta(1)$, the EMA floor $\Theta(\Sigma_{\xi})$ persists, and Eq.~\eqref{eq:g_aux_cov} is manifestly $O(t^{2})$.
\end{proof}

\textbf{Interpretation.} Two distinct, complementary variance-reduction mechanisms compose at the gradient level. First, the corrective branch is structurally EMA-free: regressing toward a positive-side target rather than a reflected EMA plug-in eliminates the $4\beta^{2}(\beta+1)^{2}\Sigma_{\xi}$ plug-in floor that the negative branch of NFT-style estimators carries by construction (Eq.~\eqref{eq:g_nft_cov}). Second, although the corrective velocity target itself has a $\Theta(1/t^{2})$-divergent per-step covariance (Eq.~\eqref{eq:aux_total_variance})---a direct consequence of $\xbarpos\!\perp\! x_t^{(i)}\mid y$, which prevents the flow-matching posterior from collapsing the marginal positive dispersion---the $t^{2}$ weighting induced by writing the loss in $\mathbf{x}_0$-space (Eq.~\eqref{eq:aux}) injects a $t^{4}$ factor into the gradient covariance (Eq.~\eqref{eq:g_aux_cov}), turning the divergent target into an $O(t^{2})$ vanishing gradient floor. Aggregation over $|\mathcal{P}|$ within-group positives further shrinks the residual gradient covariance by exactly $|\mathcal{P}|^{-1}$. Both effects rely only on the rollout group's i.i.d.\ structure under shared conditioning $y$ and on the loss being formulated in $\mathbf{x}_0$-space; no covariance-dominance or matching-Lipschitzness condition is required.

\textbf{Takeaway.} \methodname{}'s corrective branch achieves a strict reduction in the parameter-gradient trace covariance for all sufficiently small $t<t^{\star}$, and asymptotically dominates the NFT branch as $t\!\to\!0$. Two complementary, unconditional mechanisms drive this reduction. First, dispensing with the reflected EMA plug-in eliminates the persistent $\Theta(\beta^{2}(\beta+1)^{2}\Sigma_{\xi})$ EMA floor that the NFT branch retains (Eq.~\eqref{eq:g_nft_cov}). Second, the $\mathbf{x}_0$-space loss formulation (Eq.~\eqref{eq:aux}) supplies an automatic $t^{2}$ down-weighting of the corrective gradient (Eq.~\eqref{eq:g_aux_cov}) that perfectly cancels the $\Theta(1/t^{2})$ divergence of the marginal-positive target covariance, leaving an $O(t^{2})$ residual that is further shrunk by $|\mathcal{P}|^{-1}$ via the within-group positive mean. Both reductions follow from the rollout group's i.i.d.\ structure under shared conditioning $y$, with no additional architectural or empirical assumptions.

\subsection{Discussion of assumptions}
\label{app:assumption_discussion}

\textbf{Reward-local mask assumption.} Equation~\eqref{eq:mask_reward_relevance_corrected} says that off the violation mask, positive and negative rollouts have the same conditional velocity in expectation. This is the formal version of the credit-assignment intuition behind the compositional monitor: a binary failure should be attributed to the entities and frames appearing in the violation trace, not to the entire video. The assumption is not exact at the pixel level, because rendering jitter and background changes can occur off mask. The KL anchor and the pretrained video prior are therefore still necessary to prevent irrelevant coordinates from drifting. Operationally, the $|\mathbf{M}|/|\Omega|\in[1\%,5\%]$ sparse-support observation of \S\ref{sec:intro} is a direct consequence: violation traces concentrate on the few entities and frames that actually disagreed with the predicate, not on the bulk of the latent.

\textbf{Population-level mask assumption (Assumption~\ref{ass:population_mask}).} Strictly speaking, the empirical mask $\mathbf{M}$ in Eq.~\eqref{eq:group_mask} is a deterministic function of the within-group un-noised rollouts $\{\mathbf{x}_0^{(j)}\}_{j=1}^{N}$ (through the LTL violation traces and the SAM3 atlases) and is therefore random when one conditions on $(x_t^{(i)},t)$ alone. Factoring $\mathbf{M}(\omega)$ out of $\mathbb{E}[\,\cdot\mid x_t,t]$ in the proofs of Theorem~\ref{thm:masked_rg_optimum} and Lemma~\ref{lem:aux_fp} therefore requires the population-level reading of $\mathbf{M}$ stated in Assumption~\ref{ass:population_mask}: at the population level we identify $\mathbf{M}$ with the reward-locality support that already appears in Eq.~\eqref{eq:mask_reward_relevance_corrected}, so $\mathbf{M}$ is treated as $\sigma(x_t,t)$-measurable and the empirical $\mathbf{M}$ of Eq.~\eqref{eq:group_mask} is read as a Monte Carlo estimator of that fixed support. This is a natural population counterpart of the reward-local mask assumption above: both express the same credit-assignment principle---that off-mask coordinates carry no clause-level disagreement---in their respective regimes. The assumption is unrelated to the variance-reduction story of \S\ref{app:variance}, which is fully unconditional.

\textbf{Mask validity across noise levels.} We emphasize that $\mathbf{M}$ is a mask on the \emph{coordinates of the velocity tensor}, not on the \emph{content of $x_t$}. Reward locality (Eq.~\eqref{eq:mask_reward_relevance_corrected}) is a population-level statement about the conditional velocity fields $v^{\pm},v^{\text{old}}$ at every $(x_t,t)$, independent of how informative $x_t$ is about $x_0$. In particular, at high noise where $x_t\approx\epsilon$ becomes uninformative about $x_0$, the conditional positive and negative velocities both collapse on $\bar{\mathbf{P}}$ to $v^{\text{old}}$ (off-mask reward locality), while on $\mathbf{P}$ they coincide with the marginal-prototype velocity $\bar v^{+}$ (Lemma~\ref{lem:aux_fp}). Masking therefore remains exactly the right localization across the entire $t\in[0,1]$ trajectory, and the mask need not be $t$-dependent.

\textbf{Within-group positive mean as the corrective target.} The corrective branch uses the empirical mean $\xbarpos=|\mathcal{P}|^{-1}\sum_{j\in\mathcal{P}}\mathbf{x}_0^{(j)}$ of within-group positives instead of any selected single positive. Because the within-group positives are i.i.d.\ from $\pi^{+}(\,\cdot\mid y)$ and independent of the negative's noised state $x_t^{(i)}$ given $y$, the conditional first moment of $\xbarpos$ is the marginal positive mean $\mathbb{E}_{\pi^{+}(\cdot\mid y)}[x_0]$ rather than the flow-matching conditional mean $\mathbb{E}_{\pi^{+}}[x_0\mid x_t,t]$. Lemma~\ref{lem:aux_fp} therefore states the population target as the marginal-prototype velocity $\bar v^{+}(x_t,t)=(x_t-\mathbb{E}_{\pi^{+}(\cdot\mid y)}[x_0])/t$, which agrees with the flow-matching positive velocity $v^{+}$ at high noise and supplies a $t$-uniform marginal-positive bias at finite $t$. The corrective branch should therefore be read as a distributional anchor toward $\pi^{+}$'s marginal, not as a flow-matching estimator of $v^{+}$; its complementarity with the credit-aware NFT branch (which carries the $t$-dependent $v^{+}$ pull) is the mechanism by which the joint update ends up positively oriented along $\Delta$ in expectation. Aggregation also shrinks the within-step target covariance by a factor of $|\mathcal{P}|^{-1}$ relative to any single-positive variant (Eq.~\eqref{eq:aux_total_variance}); after multiplication by the $\mathbf{x}_0$-space loss's $t^{4}$ gradient weighting, this $|\mathcal{P}|^{-1}$ factor carries through to the corrective gradient covariance Eq.~\eqref{eq:g_aux_cov}. The only residual statistical assumption is that $|\mathcal{P}|\ge 1$ within a group, which the credit-aware NFT branch already requires for the violation-trace mask $\mathbf{M}$ to be non-degenerate.

\textbf{KL term.} The KL regularizer generally moves the fixed point. Only under the local quadratic approximation of Eq.~\eqref{eq:kl_quadratic_theory}, and only when $v_{\text{ref}} = v^{\text{old}}$, does it simply shrink the coordinate-wise step without rotating the update direction. We therefore state the full weighted optimum of Proposition~\ref{prop:direction} with the KL term included, rather than claiming that KL ``does not move'' the fixed point.

\textbf{Gradient-level SNR.} By Theorem~\ref{thm:masked_rg_optimum} the credit-aware NFT loss has zero gradient on off-mask coordinates, and under reward locality (Eq.~\eqref{eq:mask_reward_relevance_corrected}) the off-mask sample noise from either branch has zero population mean, so masking is unbiased. Translating this into a gradient-level statement at the parameter level additionally requires the Jacobian block-alignment condition Eq.~\eqref{eq:jacobian_block_alignment} of \S\ref{app:variance}, which prevents off-mask residual noise from leaking into on-mask parameters through shared $J^{\top}$. For the flow-matching DiT architecture used in this work, spatial-temporal parameter sharing is weak across the mask boundary because the mask is a function of action-relevant content that the attention layers localize, so the approximation holds well in practice.

\textbf{Corrective-reflow gradient variance.} The variance-reduction claim of \S\ref{app:variance} is stated at the parameter-gradient level rather than at the velocity-target level, and is fully unconditional under the rollout group's i.i.d.\ structure. The corrective velocity target $Z_{\text{CR}}^{(i)}$ has a $\Theta(1/t^{2})$-divergent per-step covariance at small $t$ (Eq.~\eqref{eq:aux_total_variance}), because $\xbarpos\!\perp\! x_t^{(i)}\mid y$ forbids the flow-matching posterior from collapsing the marginal positive dispersion; we make no claim of variance reduction at the target level. At the gradient level, however, the $\mathbf{x}_0$-space formulation of the corrective loss (Eq.~\eqref{eq:aux}) injects a $t^{2}$ down-weighting that converts this divergent target into an $O(t^{2})$-vanishing gradient floor (Eq.~\eqref{eq:g_aux_cov}), while the NFT-branch gradient retains a persistent $\Theta(\beta^{2}(\beta+1)^{2}\Sigma_{\xi})$ EMA-plug-in floor (Eq.~\eqref{eq:g_nft_cov}). Aggregation over $|\mathcal{P}|$ within-group positives further shrinks the residual corrective-branch gradient covariance by a factor of $|\mathcal{P}|^{-1}$. Proposition~\ref{prop:variance} therefore guarantees a strict reduction $\mathrm{tr}\,\mathrm{Cov}[g_{\text{NFT}}]>\mathrm{tr}\,\mathrm{Cov}[g_{\text{CR}}]$ for all $t<t^{\star}$, and the gap is $\Theta(1)$ as $t\!\to\!0$. No covariance-dominance or matching-Lipschitzness condition is required.

\textbf{Weighted-ERM generalization.} The corrective loss admits a one-parameter family of distributional-matching variants
\begin{equation}
  \Lcr^{(w)}(\theta) \;=\; \lambda_{\text{CR}}\,\mathbb{E}_{i\in\mathcal{N},t,\epsilon}\!\left[\,\sum_{j\in\mathcal{P}} w_{ij}\,\big\|\mathbf{M}\odotm(\xhat_0^{(i)}-\mathbf{x}_0^{(j)})\big\|^{2}\,\right], \quad \sum_{j\in\mathcal{P}}w_{ij}=1,\ w_{ij}\ge 0,
  \label{eq:aux_weighted}
\end{equation}
which are gradient-equivalent (by the squared-loss identity $\sum_j w_{ij}\|x-x_j\|^2=\|x-\sum_j w_{ij}x_j\|^2+\text{const}$) to regression toward the weighted barycenter $\sum_j w_{ij}\mathbf{x}_0^{(j)}$. The uniform-weight choice $w_{ij}=|\mathcal{P}|^{-1}$ used by the main method recovers Eq.~\eqref{eq:aux}; the kernel-weight choice $w_{ij}\propto\exp(-\|\actphi(\mathbf{x}_0^{(i)})-\actphi(\mathbf{x}_0^{(j)})\|^{2}/\tau)$ with bandwidth $\tau$ recovers nearest-positive matching as $\tau\to 0$ and the uniform mean as $\tau\to\infty$. We treat the kernel-weighted variant as an ablation (App.~\ref{app:variance}) rather than the main method, since the uniform mean is the closed-form Monte Carlo estimator of the marginal positive prototype that drives Lemma~\ref{lem:aux_fp} and Proposition~\ref{prop:variance}.

\textbf{Summary.} \methodname{} is best read as a localized corrective estimator for sparse-reward video diffusion RL. Violation masking projects the reward-induced learning signal onto the failure-relevant coordinates---retaining the full positive-flow pull along the reward-induced direction $\Delta$ where the reward actually disagrees with the rollout, and exactly cancelling both branches' off-mask sample noise under reward locality (Theorem~\ref{thm:masked_rg_optimum}, Eqs.~\eqref{eq:offmask_pos_var}--\eqref{eq:offmask_neg_var}). The corrective loss replaces the reflected EMA plug-in with the within-group positive mean $\xbarpos$, a distributional anchor whose conditional first moment is the marginal positive prototype $\mathbb{E}_{\pi^{+}(\cdot\mid y)}[x_0]$ on the violation mask (Lemma~\ref{lem:aux_fp})---coinciding with the flow-matching positive velocity $v^{+}$ at high noise and supplying a $t$-uniform marginal-positive pull at finite $t$---and whose $\mathbf{x}_0$-space loss formulation yields a strictly lower-variance \emph{parameter-gradient} estimator at small $t$ via an automatic $t^{2}$ down-weighting and an additional $|\mathcal{P}|^{-1}$ within-group shrinkage (Proposition~\ref{prop:variance}). Under a local quadratic KL approximation, the active-coordinate update of \methodname{} superposes a $\Delta$-aligned NFT pull and a $\bar\Delta$-aligned corrective pull on $\mathbf{P}$, with a coordinate-dependent step length on the violation-traced support (Proposition~\ref{prop:direction}). Together, these two estimator-level interventions improve the bias--variance tradeoff of the sparse-reward update and explain the faster, more stable optimization observed empirically in \S\ref{sec:exp}.

\end{document}